\documentclass[letterpaper,11pt]{article}
\usepackage[margin=1in]{geometry}
\usepackage[utf8]{inputenc} % allow utf-8 input
\usepackage[T1]{fontenc}    % use 8-bit T1 fonts

\usepackage{macros}

\title{Boosting and the Expressive Power of Simple Weak Learners via the $\gamma$-VC Dimension}

\author{
    \begin{minipage}[t]{0.3\textwidth}
        \centering
        {\normalsize Arthur da Cunha} \\
        {\small Aarhus University} \\
        {\small dac@cs.au.dk}
    \end{minipage}
    \hfill
    \begin{minipage}[t]{0.3\textwidth}
        \centering
        {\normalsize Kasper Green Larsen} \\
        {\small Aarhus University} \\
        {\small larsen@cs.au.dk}
    \end{minipage}
    \hfill
    \begin{minipage}[t]{0.3\textwidth}
        \centering
        {\normalsize Liang-Yu Zou} \\
        {\small Aarhus University} \\
        {\small zou@cs.au.dk}
    \end{minipage}
}

\date{\vspace*{-1em}}

\begin{document}

% \pagenumbering{gobble}          % Suppress page numbers (no page number displayed)

\maketitle

\begin{abstract}
    Boosting converts weak hypotheses with a small edge over random guessing into highly accurate predictors, but the expressive power of the resulting classifier can depend strongly on the structure of the base class.
    We study this phenomenon through the $\gamma$-VC dimension introduced by Alon et al. (STOC 2021). Our first result shows that this parameter characterizes the sample complexity for weak-to-strong learning up to a constant factor scaling in $\gamma$. We then sharpen the general relationship between the classic VC dimension and the $\gamma$-VC dimension. Finally, we also give improved upper and lower bounds on the $\gamma$-VC dimension for the fundamental concept classes of decision stumps and axis-parallel rectangles in $\mathbb{R}^d$.

\end{abstract}

% \newpage

% \tableofcontents

% \newpage

\pagenumbering{arabic}
\setcounter{page}{1}

\section{Introduction}

Boosting is one of the central paradigms in learning theory.
In its classical formulation, one is given access to a \emph{weak learner}: an algorithm which, for every distribution over labeled examples generated by an unknown target concept, returns a hypothesis whose error is bounded away from $1/2$ by some edge parameter $\gamma/2\in (0, 1/2)$.
The goal is to convert such a weak learner into a \emph{strong learner}, whose error can be made arbitrarily small.
The question of whether weak learnability and strong learnability are equivalent was raised by~\citet{kearns1988learning, kearns1994cryptographic} and answered affirmatively by~\citet{schapire1990strength}.
This line of work culminated in the AdaBoost algorithm of~\citet{freund1997decision}, which remains one of the most influential algorithms in machine learning.

At a high level, boosting works by repeatedly training weak classifiers and combining their predictions into a single final classifier.
Each weak classifier may only be slightly better than random guessing, but the training procedure adaptively focuses attention on examples that previous classifiers handled poorly.
The final prediction is then obtained by a weighted vote of the weak classifiers, or sometimes using more elaborate aggregation rules.
Thus, boosting transforms a collection of simple, moderately accurate hypotheses into a highly accurate aggregate predictor.

A key motivation for boosting is that the weak hypotheses are intended to be simple ``rules of thumb''.
In both theory and practice, the base class is often taken to be a very simple and efficiently learnable class, such as decision stumps, small decision trees, or other low-complexity classifiers.
Decision stumps are particularly canonical: they depend on a single feature and a single threshold, and already in early experimental evaluations of AdaBoost they were used as a standard base hypothesis class~\citep{FreundSchapire96, FreundSchapire99}.
From this perspective, the central question is not merely whether weak learners can be boosted, but what can be learned by boosting hypotheses from a fixed simple base class.

A fundamental question in weak-to-strong learning is how many samples are needed to carry out this transformation.
\citet{larsen2022optimal} gave a sharp answer in terms of the VC dimension of the base class: if the base class $\cH$ has VC dimension $\VC(\cH)$, then there is a boosting algorithm using
\begin{align}
\label{eq:vcworstcase}
    O\left(
        \frac{\VC(\cH)}{\eps \gamma^2}
        +
        \frac{\log(1/\delta)}{\eps}
    \right)
\end{align}
samples to obtain error at most $\eps$ with confidence $1 - \delta$.
Their upper bound is obtained by combining AdaBoost with a subsampling technique of~\citet{hanneke2016optimal}, or with Bagging~\citep{breiman1996bagging,baggingLarsen}.
Moreover, \citet{larsen2022optimal} showed that for every integer $d = \Omega(\log(1/\gamma))$, there exists a base hypothesis class $\cH$ with $\VC(\cH) \le d$ for which the sample complexity bound is tight.

This result gives a complete worst-case answer when the only parameter of the base class is its VC dimension.
However, their lower bound only holds for some specifically designed classes, leaving room for improvements for very simple (low VC dimension) base classes. 
Furthermore, it does not capture the more refined structure of a particular base hypothesis class.
In particular, VC dimension measures the difficulty of learning the base hypothesis class itself, while boosting is concerned with the expressive power obtained by aggregating many hypotheses from the base class, most often via majority voting.
This distinction holds particular significance for simple base classes in the boosting context, such as decision stumps, whose ordinary VC dimension is small but whose boosted expressivity can be substantially richer.

The expressivity question was introduced and studied by~\citet{alon2023boosting}, who introduced the concept of \emph{$\gamma$-VC dimension}.
% Given a base class $\cH \subseteq \{\pm 1\}^{\cX}$, a labeled sample is called $\gamma$-realizable by $\cH$ if, for every distribution on the sample, some $h \in\cH$ has correlation at least $\gamma$ with the labels.
% A set of points is $\gamma$-interpolated by $\cH$ if every binary labeling of the set is $\gamma$-realizable by $\cH$.
% The $\gamma$-VC dimension, denoted $\VC_\gamma(\cH)$, is the largest size of a $\gamma$-interpolated set.

\begin{definition}[$\gamma$-VC dimension]
    Let $\gamma \in (0, 1]$.
    Let $\cH \subseteq \{\pm 1\}^\cX$ be a hypothesis class.
    A sample $S = ((x_1, y_1), \dotsc, (x_n, y_n))$ is \underline{$\gamma$-realizable} with respect to $\cH$ if for any probability distribution $\cQ$ over $S$, there exists a hypothesis $h \in \cH$ such that,
    \[
    \corr_\cQ(h)
    \coloneqq \bbE_{(x, y) \sim \cQ}[y \cdot h(x)]
    \ge \gamma
    .
    \]
    A distribution $\cD$ over $\cX \times \{\pm 1\}$ is called \underline{$\gamma$-realizable} by $\cH$ if any i.i.d. sample drawn from $\cD$ is $\gamma$-realizable with respect to $\cH$.

    A set $\{x_1, \dotsc, x_d\} \subseteq \cX$ is \underline{$\gamma$-interpolated} by $\cH$ if for any concept $c\colon \cX \to \{\pm 1\}$, the sample $S = ((x_1, c(x_1)), \dotsc, (x_d, c(x_d)))$ is $\gamma$-realizable with respect to $\cH$.

    The \underline{$\gamma$-VC dimension} of $\cH$, denoted by $\VC_\gamma(\cH)$, is the maximum size of a set which is  $\gamma$-interpolated by $\cH$.
    If there exist sets of arbitrary large size that are $\gamma$-interpolated by $\cH$, then we define $\VC_\gamma(\cH) = \infty$.
\end{definition}

For $\gamma = 1$, this recovers the usual VC dimension.
For smaller $\gamma$, it measures how rich a family of target concepts can be weakly represented, with edge $\gamma$, by the base class. Besides, $\bbE_{(x, y) \sim \cQ}[y \cdot h(x)]\ge \gamma$ is equivalent to $\bbP_{(x, y) \sim \cQ}[h(x) \ne y] \le (1 - \gamma)/2$, so $\gamma$-realizability is equivalent to weak learnability with edge $\gamma$.

This parameter not only measures expressivity; it also characterizes the sample complexity of weak-to-strong learning.
To see the lower bound, let $P$ be a set of size $\VC_\gamma(\cH)$ that is $\gamma$-interpolated by $\cH$. Then every labeling $c\colon P \to\{\pm 1\}$ is $\gamma$-realizable with respect to $\cH$: for every distribution $\cD$ over $P$, there is an $h \in\cH$ such that
\[
    \bbE_{x \sim\cD}[c(x)h(x)]
    \ge \gamma
    .
\]
Thus the class of all labelings of $P$ is weakly learnable with edge $\gamma$ using the base class $\cH$.
Any weak-to-strong learner that succeeds for all such weakly learnable targets must therefore, in particular, PAC learn the full class $\{\pm 1\}^P$.
The standard VC lower bound for realizable PAC learning implies that
\[
    \Omega\left(
        \frac{\VC_\gamma(\cH) + \log(1/\delta)}{\eps}
    \right)
\]
samples are necessary. For completeness, we have included a formal and self-contained proof of this lower bound in Section~\ref{sec:weak_to_strong-lower_bound}.

Prior work gives a matching information-theoretic upper bound.
\citet{partial} developed a theory of PAC learning for partial concept classes, showing that finite VC dimension of partial concept classes characterizes learnability in this setting.
Subsequent work of~\citet{pacwithoutuniform} gave optimal PAC bounds for partial concept classes using aggregations of the one-inclusion graph algorithm.
Applied to the $\gamma$-\emph{margin} partial class associated with the base hypothesis class $\cH$, this yields sample complexity
\[
    O\left(
        \frac{\VC_\gamma(\cH) + \log(1/\delta)}{\eps}
    \right)
    .
\]
However, this route goes through general-purpose algorithms for partial concept classes, based on the one-inclusion graph, which are transductive and can be computationally impractical for the classes arising from boosting.
Moreover, such an algorithm based on partial concept classes would need full knowledge of the base class and cannot be implemented solely from oracle access to a weak learner. We refer the reader to~\citep{partial} for a formal definition of partial concept classes and the $\gamma$-margin partial class.

\paragraph{Boosting and $\gamma$-VC Dimension.}
Our first result shows that, up to a constant factor scaling in $\gamma$, the optimal weak-to-strong learning sample complexity above can be achieved by a boosting-style algorithm using only a weak learner.
This is important from an algorithmic perspective: boosting does not require knowledge of the rich $\gamma$-margin partial concept class associated with a simple base class $\cH$, but only repeatedly calls a procedure for finding simple hypotheses with nontrivial advantage and aggregates their predictions.
This is the same computational principle that makes boosting efficient in practice.
More precisely, we prove a uniform convergence theorem for \emph{large-margin} voting classifiers in which the usual VC-dimension dependence is replaced by $\VC_{C\gamma}(\cH)$, for a universal constant $C\in (0, 1)$.
Plugging this into the weak-to-strong framework developed in~\citep{larsen2022optimal} and combined with Bagging in~\citep{baggingLarsen}, yields the following guarantee.

\begin{theorem}
    \label{thm:algo}
    There is a universal constant $C \in (0, 1)$ such that the following holds.
    Suppose $\cW$ is a $\gamma$-weak learner with base class $\cH$, and let $d_{C\gamma} = \VC_{C\gamma}(\cH)$.
    Then the {\upshape\textsc{BaggedAdaBoost}} algorithm, using $\cW$, satisfies the following guarantee. For every distribution $\cD$ over $\cX \times \{\pm 1\}$ that is $\gamma$-realizable by $\cH$, and every $\delta \in (0, 1)$, given $n$ i.i.d. samples from $\cD$, the algorithm returns a classifier with error at most
    \[
        O\left(
            \frac{d_{C\gamma} + \log(1/\delta)}{n}
        \right)
    \]
    with probability at least $1 - \delta$. Equivalently, error at most $\eps$ is achieved with $O\left(\frac{d_{C\gamma} + \log(1/\delta)}{\eps}\right)$ samples.
\end{theorem}

As \emph{margins} are central to our analysis and algorithm, we now formally define margins. Given a weighted vote of weak classifiers $f = \sum_i w_i h_i$ for which $h_i  \in \{\pm 1\}^\cX$ for an input domain  $\cX$, $w_i \ge 0$ and $\sum_i w_i = 1$, the margin of the final voting classifier $\hat{f} = \sign(f)$ on an example $(x, y) \in \cX \times \{-1,1\}$ is defined as $y \cdot f(x)$, which can be thought of as a measure of confidence in the prediction. Classic algorithms like AdaBoost have been shown to produce voting classifiers with $\Omega(\gamma)$ margins on all training samples after sufficiently many iterations~\citep{boostingbook}. Large margins may in turn be used to prove generalization, see e.g.~\citep{SFBL98,breiman1999prediction,gz13,DBLP:conf/colt/HogsgaardL25,DBLP:journals/corr/abs-2511-20407}. Large margin voting classifiers and a set of samples $S$ being $\gamma$-realizable are intimately connected through von-Neumann's minimax principle. In particular, if $S$ is $\gamma$-realizable by $\cH$, then there exists a voting classifier over weak classifiers from $\cH$ with margins $\gamma$ on all samples in $S$. 

Theorem~\ref{thm:algo} strongly motivates the study of the $\gamma$-VC dimension and how large it is for various concept classes.
Moreover, as mentioned above, the algorithm \textsc{BaggedAdaBoost} (
see~\cref{alg:bagged}) is extremely efficient.
It samples $O(\ln(n/\delta))$ independent subsamples $S_t$ from $S$.
Each $S_t$ consists of $n$ independent samples with replacement from $S$.
It then runs the classic AdaBoost algorithm on each $S_t$ until all margins have reached $\Omega(\gamma)$ (which happens after $O(\gamma^{-2} \log n)$ calls to the weak learner~\citep{boostingbook}) and combines the resulting voting classifiers via a majority vote.

\begin{algorithm}
  \DontPrintSemicolon
  \KwIn{Training set $S = ((x_1, y_1), \dotsc, (x_n, y_n))$.
  Weak learner $\cW$.
  Failure probability $0 < \delta < 1$.
  }
  \KwResult{Hypothesis $g : \cX \to \{-1, 1\}$.}
    
    \For{$t = 1, \dotsc, O(\ln(n/\delta))$}{
        Let $S_t$ be a set of $n$ independent samples with replacement from $S$.
    
        Run AdaBoost on $S_t$ with $\cW$ until all margins are $\Omega(\gamma)$.
        Let $f_t : \cX \to \{-1, 1\}$ be the resulting voting classifier.
    }

    $\hat{f} \gets \sign(\sum_t f_t)$.

    \Return{$\hat{f}$}

  \caption{\textsc{BaggedAdaBoost}($S, \cW, \delta$)}\label{alg:bagged}
\end{algorithm}

\paragraph{VC-Dimension and $\gamma$-VC Dimension.}
A fundamental goal is to understand the relationship between the classic VC dimension and the $\gamma$-VC dimension.
Here~\citet{alon2023boosting} proved the following upper bound\footnote{Their paper states the bound without the last $\log(1/\gamma)$ factor, but inspecting their proof \citep[Page 34]{alon2023boosting}, the actual dependency has the additional $\log(1/\gamma)$ factor.} for all concept classes
\[
    \VC_\gamma(\cH)
    = O\left(
        \frac{\VC(\cH)}{\gamma^2}\log\frac{\VC(\cH)}{\gamma} \log \frac{1}{\gamma}
    \right)
    ,
\]
They also showed that this is nearly tight for some classes.
More precisely, for every $s\in \bbN$, there is a hypothesis class $\cH$ with $\VC(\cH) = O(s\log(1/\gamma))$ such that
\[
\VC_\gamma(\cH)
    = \Omega\left(
        \frac{s}{\gamma^2}
    \right)
    = \Omega\left(
        \frac{\VC(\cH)}{\gamma^2 \log\frac{1}{\gamma}}
    \right)
    ,
\]
Our next contribution is to remove the logarithmic factors from both the upper and lower bounds.

\begin{theorem}
    \label{thm:tightVCVCgamma}
    Let $\cH \subseteq \{\pm 1\}^{\cX}$ be a hypothesis class.
    Then, for every $\gamma \in (0, 1]$,
    \[
        \VC_\gamma(\cH)
        = O\left(\frac{\VC(\cH)}{\gamma^2}\right)
        .
    \]
    Moreover, for every $d = \Omega(\log(1/\gamma))$, there is a hypothesis class $\cH$ with $\VC(\cH) \le d$ such that
    \[
        \VC_\gamma(\cH)
        = \Omega\left(\frac{d}{\gamma^2}\right)
        .
    \]
\end{theorem}

A natural question is whether the lower bound on $\VC(\cH)$ in terms of $\gamma$ in the lower bound above is necessary.
Perhaps surprisingly it is, the requirement cannot be simply dropped.
\citet{alon2023boosting} showed that for any base hypothesis class $\cH$ with $\VC(\cH) = d$, the following alternative upper bound on $\VC_\gamma(\cH)$ also holds
\[
\VC_\gamma(\cH) = O_d\left(\left(\frac{1}{\gamma}\right)^{2-\frac{2}{d+1}}\right),
\]
where $O_d(\cdot )$ hides factors depending only on $d$.
For every fixed $d$, the bound above implies $\VC_\gamma(\cH) = o(\gamma^{-2})$. Consequently, the lower bound $\Omega(d/\gamma^2)$ cannot hold for all $d$ and $\gamma$.

In light of Theorem~\ref{thm:tightVCVCgamma}, the fastest possible growth of $\VC_\gamma(\cH)$ as a function of $d$ and $\gamma$ is, up to constants, $d/\gamma^2$.
Combined with Theorem~\ref{thm:algo}, this recovers the VC-based worst-case sample complexity bound for weak-to-strong learning from~\eqref{eq:vcworstcase}.
However, the point of introducing $\VC_\gamma$ is that it can be much smaller for specific weak-learner classes.
This leads to a more refined question: what is $\VC_\gamma(\cH)$ for the simple base classes often used in boosting?

\paragraph{Expressivity of Simple Concept Classes.}
We first study the prototypical class of decision stumps in $\bbR^d$.
\citet{alon2023boosting} proved
\[
    \Omega(1/\gamma)
    \le \VC_\gamma(\DS_d)
    \le O(d/\gamma)
    ,
\]
where $\DS_d$ denotes the class of decision stumps in $\bbR^d$.
These bounds determine the dependence on $\gamma$, but leave open the exact dependence on the ambient dimension $d$.
For the standard VC dimension, the dependency is $\VC(\DS_d) = \Theta(\log d)$~\citep{shalev2014understanding}, so one might expect that $\VC_\gamma(\DS_d)$ has only a mild dependency on $d$.

We show that the behavior is instead very different.
For $d \le 1/\gamma^2$, it holds up to polylogarithmic factors, that the $\gamma$-VC dimension of decision stumps grows as $\sqrt{d}/\gamma$.
This gives a surprising example where $\VC_\gamma$ has an exponential dependency on the $\VC$ dimension of the base class, but only an inverse linear dependency on the edge $\gamma$ instead of the standard inverse quadratic dependency in Theorem~\ref{thm:tightVCVCgamma}.

\begin{theorem}
    \label{thm:stumps-upper}
    For every integer $d \ge 1$ and every $\gamma \in (0, 1]$,
    \[
        \VC_\gamma(\DS_d)
        = O\left(\min\left\{
            \frac{\sqrt{d}}{\gamma}
            \log\!\left(\frac d\gamma\right)
        , \frac{\log d}{\gamma^2}\right\}\right)
        .
    \]
\end{theorem}

\begin{theorem}
    \label{thm:stumps-lower}
    For every integer $d \ge 1$ and every $\gamma \in (0, 1]$:
    \begin{enumerate}
        \item If $d \le 1/\gamma^2$, then
        \[
            \VC_\gamma(\DS_d)
            = \Omega\left(\frac{\sqrt{d}}{\gamma}\right)
            .
        \]
        \item If $d \ge 1/\gamma^2$, then
        \[
            \VC_\gamma(\DS_d)
            = \Omega\left(\max\left\{\frac{\log(\gamma d)}{\gamma}, \frac{1}{\gamma^2}\right\}\right)
            .
        \]
    \end{enumerate}
\end{theorem}

Our last results concern axis-parallel rectangles.
Here we let $\cR_d$ denote the base hypothesis class containing two hypotheses $h_R^+$ and $h_R^-$ for every axis-parallel rectangle $R$ in $\bbR^d$.
The hypothesis $h_R^+$ has $h_R^+(x) = 1$ for every $x  \in R$ and $h_R^+(x) = -1$ for all other points.
The hypothesis $h_R^-$ is simply $-h_R^+$ (\citet{alon2023boosting} always assume that a base class $\cH$ is symmetric so that $-h \in \cH$ for every $h\in \cH$).
For the standard VC dimension, we have $\VC(\cR_d) = \Theta(d)$. Notice also that a decision tree classifier may be thought of as a collection of axis-parallel rectangles, one for each leaf of the decision tree.

Similarly to decision stumps, we show that the worst case $\VC_\gamma(\cR_d) = O(\VC(\cR_d)/\gamma^2)$ bound from Theorem~\ref{thm:tightVCVCgamma} is very loose in its dependency on $\gamma$.
Concretely, we show the following upper bound.

\begin{theorem}
    \label{thm:intro-rectangles-upper}
    For every fixed integer $d \ge 2$ and every $\gamma \in (0, 1]$,
    \[
        \VC_\gamma(\cR_d)
        = O_d\left(
            \frac{\log^{d-1/2}(1/\gamma)}{\gamma}
        \right)
        .
    \]
\end{theorem}
We complement this with a lower bound showing a similar $\gamma^{-1} \polylog(1/\gamma)$ growth
\begin{theorem}

    \label{thm:intro-rectangles-lower}
    For every fixed integer $d \ge 2$ and every $\gamma \in (0, 1]$,
    \[
        \VC_\gamma(\cR_d)
        = \Omega_d\left(
            \frac{\log^{(d-1)/2}(1/\gamma)}{\gamma}
        \right)
        .
    \]
\end{theorem}
In summary, our results show that $\VC_\gamma$ is the right information-theoretic parameter for weak-to-strong learning, give a boosting-style algorithm matching this parameter up to a constant-factor change in the margin, sharpen the general upper bound on $\VC_\gamma$, and characterize its behavior for two natural families of simple geometric weak learners.
The results also reveal a striking phenomenon: classes such as decision stumps have logarithmic ordinary VC dimension in terms of the ambient dimension, yet their boosted expressivity can grow polynomially with the ambient dimension.

\subsection{Overview of Techniques}
The algorithm, \textsc{BaggedAdaBoost}, originates in~\citep{baggingLarsen} and is a combination of the heuristic bagging~\citep{breiman1996bagging} (bootstrap aggregation) with AdaBoost. Our new analysis follows the weak-to-strong learning framework of~\citet{larsen2022optimal}. The crucial insight from that work is as follows: Given a base hypothesis class $\cH$ and margin $\gamma$, if one can show that given $m_{\delta}$ i.i.d.\ samples $S$ from a distribution $\cD$, it holds with probability $1-\delta$ that \emph{every} voting classifier $h$ with margins at least $\gamma$ on all samples in $S$, satisfies $\er_\cD(h) \leq 1/200$ (or more generally a sufficiently small constant), then one can use bagging, or another subsampling technique by~\citet{hanneke2016optimal}, to obtain a (strong) classifier $f$ that with probability $1-\delta$ has $\er_\cD(f) = O(m_\delta/n)$ when given $n$ samples. The framework by~\citet{larsen2022optimal} proceeds to show that $m_\delta = O(\VC(\cH)/\gamma^2 + \log(1/\delta))$ suffices, resulting in the tight VC-based upper bound in~\eqref{eq:vcworstcase}. Our main technical contribution in the analysis of \textsc{BaggedAdaBoost} is to show that $m_\delta = O(\VC_{C\gamma}(\cH) + \log(1/\delta))$ samples suffice for the same $\er_{\cD}(h) \leq 1/200$ guarantee. The proof of this claim is based on a random sparsification of voting classifiers with margins $\gamma$ on a set of $m_\delta$ samples. Invoking both concentration and anti-concentration results shows that most samples have margins at least $C \gamma$ after sparsification. Disregarding the small fraction of low margin samples allows us to use the $C\gamma$-VC dimension together with a Sauer-Shelah type counting argument to prove generalization.

We next establish the general upper bound
\[
    \VC_\gamma(\cH) = O(\VC(\cH)/\gamma^2).
\]
The proof is short and uses empirical Rademacher complexity.
If
$S = \{x_1, \dotsc, x_n\}$ is $\gamma$-interpolated, then for every Rademacher
sequence $\sigma\in\{\pm1\}^n$ there is an $h\in\cH$ with
\[
    \frac{1}{n}\sum_{i=1}^n \sigma_i h(x_i) \ge \gamma .
\]
Taking expectations over $\sigma$ gives a lower bound $\gamma$ on the empirical
Rademacher complexity of $\cH$ on $S$.
Standard bounds for VC classes~\citep{dudley1967sizes, haussler1995sphere} upper bound this quantity by $O(\sqrt{\VC(\cH)/n})$, which gives
$n = O(\VC(\cH)/\gamma^2)$.
This removes the logarithmic factor in the
upper bound of~\citet{alon2023boosting}.
The
matching lower bound for any $d = \Omega(\log(1/\gamma))$ is obtained by using the probabilistic construction of~\citet{gronlund2019margin}.
Their construction produces small classes whose convex hull contains a
large-margin realization of any fixed labeling with high probability.
A union bound over all labelings
then yields a class of VC dimension at most $d$ that $\gamma$-interpolates
$\Omega(d/\gamma^2)$ points.

Let us next give a brief overview of the main ideas behind our bounds on $\VC_\gamma(\cH)$ for various concrete base hypothesis sets $\cH$. A recurring
theme is that $\gamma$-interpolation can be viewed as a discrepancy-type
condition.
Indeed, if a point set $P$ is $\gamma$-interpolated by a base class
$\cH$, then for every coloring $c\colon P\to\{\pm1\}$ and every distribution
on $P$, some hypothesis in $\cH$ has correlation at least $\gamma$ with
$c$.
In particular, applying this to the uniform distribution shows that every
coloring must have a large imbalance on at least one range induced by
$\cH$.
Thus upper bounds on the relevant discrepancy of the induced set
systems translate directly into upper bounds on $\VC_\gamma(\cH)$.
This connection was observed by~\citet{alon2023boosting}.
In particular, we use a variant of combinatorial discrepancy called \emph{signed discrepancy}, which is more suitable for our analysis.
For a system $\cF$ of subsets of a finite set $P$, the signed discrepancy of $\cF$ is
\[
\discpm(\cF) = \min_{c\colon P \to \{\pm 1\}} \max_{F \in \cF} \abs*{\sum_{p \in F} c(p) - \sum_{p \notin F} c(p)},
\]
which can be shown is at most three times the discrepancy when $\cF$ contains $P$.

For decision stumps, the general VC-based bound is far from tight.
The key upper-bound step is to relate $\VC_\gamma(\DS_d)$ to the signed discrepancy of the set system induced by decision stumps on an arbitrary point set $P\subseteq\bbR^d$. 
Each coordinate induces a permutation of the points obtained by sorting them along that coordinate, and each decision stump produces a prefix or suffix of one of these permutations.
Thus, the set system induced by decision stumps is a subfamily of the set system consisting of all intervals in these permutations.
By the discrepancy bound for intervals in permutations due to~\citet{spencer2001discrepancy}, the signed discrepancy of the set system induced by decision stumps is $O(\sqrt{d}\log n)$. 
Since $\gamma$-interpolation would force signed discrepancy at least
$\gamma n$, this yields
\[
    \VC_\gamma(\DS_d)
        = O\!\left(\frac{\sqrt{d}}{\gamma}
              \log\!\left(\frac d\gamma\right)\right).
\]

The lower bounds for decision stumps are constructive.
For the
$\Omega(\log(\gamma d)/\gamma)$ bound, we partition the point set into
$1/(2\gamma)$ blocks and use coordinates to realize all sign patterns inside
one block while being zero elsewhere.
For the stronger
$\Omega(\sqrt{d}/\gamma)$ lower bound in the range $d \le 1/\gamma^2$, we use
Hadamard matrices.
The points are arranged in blocks, each block being a
scaled copy of a Hadamard matrix.
For any signed measure induced by a labeling
and a distribution, one block has large $\ell_1$ mass; orthogonality of the
Hadamard rows then gives a row with correlation $\Omega(\gamma)$.
Finally,
that row is simulated, up to a constant factor, by a convex combination of a
constant number of decision stumps.
This proves that every labeling is
weakly realizable with edge $\gamma$.

For axis-parallel rectangles, the upper bound is based on 
discrepancy.
The relevant problem is precisely the
high-dimensional version of Tusn\'ady's problem: bounding the discrepancy of
axis-parallel boxes on arbitrary point sets.
We use the best known upper bound~\citep{nikolov2017tighter,matouvsek2020factorization} for this problem.
As for decision stumps,
we pass from ordinary discrepancy to signed discrepancy, which gives signed discrepancy
$O_d(\log^{d-1/2} n)$, and hence
\[
    \VC_\gamma(\cR_d)
        = O_d\!\left(\frac{\log^{d-1/2}(1/\gamma)}{\gamma}\right).
\]

The lower bound for rectangles adapts Roth's classical method for geometric
discrepancy~\citep{roth1954irregularities}, in a discretized form~\citep{chazelle1997lower}.
We construct a point set for which, for every signed measure on the points,
some anchored box captures mass
$
\Omega_d\!\left(\frac{\log^{(d-1)/2} n}{n}\right).
$
Combining this with $\gamma$-interpolation, it follows that
\[
    \VC_\gamma(\cR_d)
        = \Omega_d\!\left(\frac{\log^{(d-1)/2}(1/\gamma)}{\gamma}\right).
\]
In the following technical sections, we have reorganized our results in comparison to this technical overview. In particular, we have moved the analysis of \textsc{BaggedAdaBoost} last as it relies most heavily on prior work. In Section~\ref{sec:generalVC}, we prove our results relating $\gamma$-VC dimension and VC dimension. In Section~\ref{sec:discrepancy} we give the necessary background on discrepancy needed for our results on decision stumps in Section~\ref{sec:decision_stumps} and axis-parallel rectangles in Section~\ref{sec:axisparallel}. Finally, we analyse \textsc{BaggedAdaBoost} in Section~\ref{sec:baggedada}.

\section{General Bounds on the \texorpdfstring{$\gamma$}{gamma}-VC dimension}
\label{sec:generalVC}
In this section, we provide the general upper and lower bounds on the $\gamma$-VC dimension of a hypothesis class in terms of its VC dimension.

\begin{theorem}[\cref{thm:tightVCVCgamma} restated, upper bound]\
    \label{thm:general-upper_bound}
    Let $\gamma \in (0, 1]$ and $\cH \subseteq \{\pm 1\}^\cX$ be a hypothesis class with VC dimension $d \ge 1$, then $\VC_\gamma(\cH) = O(d/\gamma^2)$.
\end{theorem}
\begin{proof}[Proof of \cref{thm:general-upper_bound}]
    Let $P = \{x_1, \dotsc, x_n\} \subseteq \cX$ be a point set of size $n$ that is $\gamma$-interpolated by $\cH$.
    By definition of $\gamma$-interpolation, and considering the uniform distribution over $P$, we have that
    \[
    \forall c\colon P \to \{\pm 1\}, \exists h \in \cH, \frac{1}{n}\sum_{i=1}^n h(x_i)c(x_i) \ge \gamma
    .
    \]
    Let $\bm\sigma = (\sigma_1, \dotsc, \sigma_n)$ be a random Rademacher vector of size $n$, i.e., $\sigma_i$ is uniformly distributed over $\{\pm 1\}$ for each $i \in [n]$.
    Then for any realization of $\bm\sigma$, if we let $c(x_i) = \sigma_i$ for each $i \in [n]$, we have that there exists $h \in \cH$ such that $(1/n)\sum_{i=1}^n \sigma_i h(x_i) \ge \gamma$.
    Therefore, by taking expectation over $\bm\sigma$, we have
    \[
    \bbE_{\bm\sigma\in \{\pm 1\}^n}\left[\sup_{h \in \cH} \frac{1}{n}\sum_{i=1}^n \sigma_i h(x_i)\right]
    \ge \gamma
    ,
    \]
    where the left-hand side is exactly the empirical Rademacher complexity of $\cH$ on $P$.
    Classical results on empirical Rademacher complexity of VC-classes~\citep{dudley1967sizes,haussler1995sphere} imply that the left-hand side is at most $O(\sqrt{d/n})$ (see \citet[Section 8.3]{Vershynin_2026} for a textbook treatment).
    Therefore, we have $\gamma = O(\sqrt{d/n})$, which implies $n = O(d/\gamma^2)$.

\end{proof}

\begin{theorem}[\cref{thm:tightVCVCgamma} restated, lower bound]
    \label{thm:general-lower_bound}
    Let $\gamma \in (0, 1/40)$.
    For every integer $d = \Omega(\log(1/\gamma))$, there exists a hypothesis class $\cH \subseteq \{\pm 1\}^\cX$ of VC dimension at most $d$ such that $\VC_\gamma(\cH) = \Omega(d/\gamma^2)$.
\end{theorem}

It suffices to consider $\gamma \in (0, 1/40)$ here. For $\gamma \in [1/40, 1]$, we trivially have $\VC_\gamma(\cH) \geq \VC(\cH)\geq \frac{\VC(\cH)}{1600\gamma^2}$.
To prove this theorem, we first quote the following result from~\citet{gronlund2019margin}, where $\conv(\cH)$ denotes the set of all weighted votes $f = \sum_i w_i h_i$ with $w_i \geq 0$ and $\sum_i w_i = 1$ with each $h_i \in \cH$.

\begin{lemma}[{\citealp[Lemma 3]{gronlund2019margin}}]\label{lem:margin-lower-bound}
    For every $\gamma \in (0, 1/40)$, $\delta \in (0, 1)$ and integers $k \le u$, there exists a distribution $\cQ = \cQ(k, u, \gamma, \delta)$ over a collection of hypothesis classes, each of which is a subset of $\{\pm 1\}^{\cX}$, where $\cX$ is a finite set of size $u$, such that
    \begin{itemize}
        \item For all $\cH \in \supp(\cQ)$, we have $|\cH| = N$;
        \item For every labeling $c \in \{\pm 1\}^\cX$, if no more than $k$ points $x \in \cX$ satisfy $c(x) = -1$, then
        \[
        \bbP_{\cH \sim \cQ}\left[\exists f \in \conv(\cH) : \forall x \in \cX, f(x)c(x) \ge \gamma\right]
        \ge 1 - \delta
        ,
        \]
        where $N = \Theta(\gamma^{-2}\log u \log(\gamma^{-2}\delta^{-1}\log u)e^{\Theta(\gamma^2k)})$.
    \end{itemize}
\end{lemma}

\begin{proof}[Proof of \cref{thm:general-lower_bound}]
    Let $d \geq C_2 \log(1/\gamma)$ be an integer for $C_2>0$ a sufficiently large constant, let $k = u = \lceil C_0 \cdot d/\gamma^2 \rceil$ and $\delta = 1/2^{u+1}$, where $C_0 > 0$ is a universal constant.
    Let $\cQ$ be the distribution over hypothesis classes given by \cref{lem:margin-lower-bound} with parameters $(k, u, \gamma, \delta)$.
    Since $k = u$, we have that the second property in \cref{lem:margin-lower-bound} works for every labeling, which implies that for every labeling $c \in \{\pm 1\}^\cX$,
    \[
    \bbP_{\cH \sim \cQ}\left[\nexists f \in \conv(\cH) : \forall x \in \cX, f(x)c(x) \ge \gamma\right]
    \le \delta
    .
    \]
    Since $|\cX| = u$ and $\delta = 1/2^{u+1}$, a union bound over all $2^u$ labelings implies that
    \begin{align*}
        &\bbP_{\cH \sim \cQ}\left[\forall c \in \{\pm 1\}^\cX : \exists f \in \conv(\cH) : \forall x \in \cX, f(x)c(x) \ge \gamma\right] \\
        &= 1 - \bbP_{\cH \sim \cQ}\left[\exists c \in \{\pm 1\}^\cX : \nexists f \in \conv(\cH) : \forall x \in \cX, f(x)c(x) \ge \gamma\right] \\
        &\ge 1 - \sum_{c \in \{\pm 1\}^\cX} \bbP_{\cH \sim \cQ}\left[\nexists f \in \conv(\cH) : \forall x \in \cX, f(x)c(x) \ge \gamma\right]\\
        &\ge 1 - 2^u \cdot \delta = \frac{1}{2}
        .
    \end{align*}
    The positive probability implies that there exists $\cH \in \supp(\cQ)$ such that for every labeling $c \in \{\pm 1\}^\cX$, there exists $f \in \conv(\cH)$ such that $\forall x \in \cX, f(x)c(x) \ge \gamma$.

    Let $\cH$ be a hypothesis class that satisfies this property.
    Fix any labeling $c \in \{\pm 1\}^\cX$, and let $f$ be the corresponding function in $\conv(\cH)$ such that $\forall x \in \cX, f(x)c(x) \ge \gamma$.
    Then for any distribution $\cD$ over $\cX$, we have $\bbE_{x \sim \cD}[f(x)c(x)] \ge \gamma$.
    Since $f \in \conv(\cH)$, there must exist a hypothesis $h \in \cH$ such that $\bbE_{x \sim \cD}[h(x)c(x)] \ge \gamma$.
    Since this holds for every labeling and every distribution, we have that $\cX$ is $\gamma$-interpolated by $\cH$.
    Therefore, $\VC_\gamma(\cH) \ge |\cX| = u = \Omega(d/\gamma^2)$.

    Finally, we show that the VC dimension of $\cH$ is at most $d$.
    Since $|\cH| = N$, the VC dimension of $\cH$ is at most $\log N$.
    By the definition of $N$ and the assumption $d \geq C_2 \log(1/\gamma)$, we have
    \[
    \log N
    \le C_1 \cdot \left(\log(1/\gamma) + C_0 \cdot d\right)
    \le C_1 \cdot \left(\frac{1}{C_2} + C_0\right) \cdot d
    \le d
    ,
    \]
    where $C_1 > 0$ is a universal constant that comes from the size of $N$, and the last inequality holds by choosing suitable $C_0, C_2$.
    This completes the proof.
\end{proof}

\section{Discrepancy Background} \label{sec:discrepancy}

In this section, we introduce some techniques from discrepancy theory, which will be used to analyse the $\gamma$-VC dimension of decision stumps and axis-parallel rectangles.

For a system $\cF$ of subsets of a finite set $P$, we define \emph{discrepancy} as follows:
\begin{itemize}
    \item For any subset $F \subseteq P$, the discrepancy of $F$ with respect to a coloring $c\colon P \to \{\pm 1\}$ is defined as $\disc(F, c) \coloneqq \abs*{\sum_{p \in F} c(p)}$.
    \item The discrepancy of $\cF$ with respect to a coloring $c$ is defined as $\disc(\cF, c) \coloneqq \max_{F \in \cF} \disc(F, c)$.
    \item The discrepancy of $\cF$ is defined as $\disc(\cF) \coloneqq \min_{c\colon P \to \{\pm 1\}} \disc(\cF, c)$.
\end{itemize}

For convenience of our analysis, we also define \emph{signed discrepancy} as follows: for any subset $F \subseteq P$, the signed discrepancy of $F$ with respect to $c$ is defined as $\discpm(F, c) \coloneqq \abs*{\sum_{p \in F} c(p) - \sum_{p \notin F} c(p)}$.
$\discpm(\cF, c)$ and $\discpm(\cF)$ are defined similarly.
When $\cF$ contains the set $P$ as one of its subsets, we can bound the signed discrepancy by the discrepancy as follows:

\begin{lemma}\label{lem:signed-discrepancy-bound}
    Let $P$ be a set of $n$ points, and $\cF$ be a system of subsets of $P$, which contains the set $P$ as one of its subsets.
    Then $\discpm(\cF) \le 3 \cdot \disc(\cF)$.
\end{lemma}

\begin{proof}[Proof of \cref{lem:signed-discrepancy-bound}]
    Let $c\colon P \to \{\pm 1\}$ be a coloring that attains the minimum of $\disc(\cF)$.
    Since $\cF$ contains $P$, $\disc(P, c) = \abs*{\sum_{p \in P} c(p)}$ is the discrepancy of $P$ with respect to $c$.
    By definition, we have
    \begin{align*}
        \disc(\cF) = \disc(\cF, c) = \max_{F \in \cF} \disc(F, c) \ge \disc(P, c).
    \end{align*}
    For every $F\in \cF$, we have
    \begin{align*}
        \discpm(F, c)
        &= \abs*{\sum_{p \in F} c(p) - \sum_{p \notin F} c(p)}\\
        &= \abs*{2 \cdot \sum_{p \in F} c(p) - \sum_{p \in P} c(p)}\\
        &\le 2 \cdot \abs*{\sum_{p \in F} c(p)} + \abs*{\sum_{p \in P} c(p)}\\
        &\le 2 \cdot \disc(F, c) + \disc(P, c)\\
        &\le 3 \cdot \disc(\cF, c)\\
        &= 3 \cdot \disc(\cF).
    \end{align*}
    Thus,
    \[
    \discpm(\cF)
    \le \discpm(\cF, c)
    \le 3 \cdot \disc(\cF),
    \]
    which concludes the proof.
\end{proof}

\section{Decision Stumps} \label{sec:decision_stumps}
In this section, we discuss the $\gamma$-VC dimension of decision stumps.
A decision stump is a simple classifier that makes a decision based on a single feature.
The class of $d$-dimensional decision stumps $\DS_d$ contains all concepts of the form
\[
h(\bx) = \begin{cases} 1 & \text{if } x_i \ge t \\ -1 & \text{otherwise} \end{cases}
\qquad\text{or} \qquad
h(\bx) = \begin{cases} -1 & \text{if } x_i \ge t \\ 1 & \text{otherwise} \end{cases}
\]
where $\bx = (x_1, x_2, \dotsc, x_d) \in \bbR^d$ is a $d$-dimensional vector and $t \in \bbR$ is the threshold.

In terms of the $\gamma$-VC dimension of $\DS_d$, \citet[Section 5.3]{alon2023boosting} provided an upper bound of $O(d/\gamma)$, and a lower bound of $\Omega(1/\gamma)$.
Therefore, the dependence on $\gamma$ is tight, but the dependence on $d$ remains open.
In the following, we tighten both the upper and lower bounds, using techniques from discrepancy theory.

\subsection{Upper Bound}
To establish the upper bound, we first bound the discrepancy of the set system induced by decision stumps. By \cref{lem:signed-discrepancy-bound}, this also yields a bound on the signed discrepancy.
Then, we apply the connection between $\gamma$-VC dimension and signed discrepancy to derive an upper bound on $\VC_\gamma(\DS_d)$.

In the following, given a $d$-dimensional point set $P \subseteq \bbR^d$, we denote by $\cDS_d(P)$ the set system induced by the class $\DS_d$ on $P$, i.e., $\cDS_d(P) = \{P\cap h^{-1}(1): h \in \DS_d\}$.

\begin{lemma}
    \label{lem:ds-discrepancy}
    For any $d$-dimensional point set $P$ of size $n$, the discrepancy of the set system $\cDS_d(P)$ satisfies $\disc(\cDS_d(P)) = O(\sqrt{d}\log n)$.
\end{lemma}

To prove \cref{lem:ds-discrepancy}, we use the following discrepancy result on permutation families from \citet{spencer2001discrepancy}.

\begin{lemma}[{\citealp[Corollary 2.1]{spencer2001discrepancy}}]
    \label{lem:ds-permutation-discrepancy}
    Let $\pi_1,\ldots,\pi_l$ be $l$ permutations of $[n]=\{1,2,\ldots,n\}$, and let $\cI$ be the set system consisting of all intervals in these permutations, i.e., $\cI = \{\{\pi_k(i),\pi_k(i+1),\ldots,\pi_k(j)\}:k\in [l], 1\leq i\leq j\leq n\}$. 
    Then
    \[
        \disc(\cI)=O\!\left(\sqrt{l}\log n \right).
    \]
\end{lemma}

\begin{proof}[Proof of \cref{lem:ds-discrepancy}]
    For each coordinate $i\in[d]$, order the points of $P$ according to their $i$-th coordinate, breaking ties arbitrarily, and let $\pi_i$ denote the resulting permutation of $P$. 
    A decision stump using the $i$-th coordinate can only produce a set of points that is a prefix or suffix of the permutation, which is also an interval.
    Let $\cI$ be the set system consisting of all intervals in the permutations $\pi_1,\ldots,\pi_d$. Then, $\cDS_d(P)\subseteq \cI$. By \cref{lem:ds-permutation-discrepancy}, we have
    \[
        \disc(\cDS_d(P))\leq \disc(\cI)=O\!\left(\sqrt{d}\log n \right).
    \]
\end{proof}

% Next, we bound the signed discrepancy of $\cDS_d(P)$ using the bound on the discrepancy of $\cDS_d(P)$.

With this in place, we can prove the following upper bound for the $\gamma$-VC dimension of $\DS_d$.

\begin{theorem}[\cref{thm:stumps-upper} restated]
    \label{thm:ds-upper_bound}
    For every integer $d \ge 1$ and every $\gamma \in (0, 1]$,
    \[
        \VC_\gamma(\DS_d)
        = O\left(\min\left\{
            \frac{\sqrt{d}}{\gamma}
            \log\!\left(\frac d\gamma\right)
        , \frac{\log d}{\gamma^2}\right\}\right)
        .
    \]
\end{theorem}

The $O(\log d/\gamma^2)$ term comes from the general upper bound on the $\gamma$-VC dimension in \cref{thm:tightVCVCgamma}, and that decision stumps in $\bbR^d$ have VC dimension $\Theta(\log d)$ \citep{shalev2014understanding}.
Thus, we only need to prove the $O(\sqrt{d}/\gamma \cdot \log(d/\gamma))$ term.

\begin{proof}[Proof of \cref{thm:ds-upper_bound}]
    Let $P$ be a set of $n$ points in $\bbR^d$ that is $\gamma$-interpolated by $\DS_d$.
    Let $\cD$ be the uniform distribution over $P$.
    Then, for any concept $c\colon P \to \{\pm 1\}$, there exists a decision stump $DS\in \DS_d$ such that
    \begin{align}
        \label{eq:ds-upper-corr}
    \corr_\cD(DS)
    = \frac{1}{n} \cdot \left(\sum_{p \in DS(P)} c(p) - \sum_{p \notin DS(P)} c(p)\right)
    \ge \gamma,
    \end{align}
    where $DS(P)$ is the subset of $P$ that is classified as $1$ by $DS$.

    On the other hand, there always exists a decision stump that classifies every point as $1$, which means $\cDS_d(P)$ always contains $P$.
    Combining \cref{lem:signed-discrepancy-bound} and \cref{lem:ds-discrepancy}, we have 
    \[
    \discpm(\cDS_d(P))
    \le 3\disc(\cDS_d(P))
    = O(\sqrt{d}\log n).
    \]
    This implies that there exists a concept $c_1\colon P \to \{\pm 1\}$ such that for any decision stump $DS$, $\abs*{\sum_{p \in DS(P)} c_1(p) - \sum_{p \notin DS(P)} c_1(p)} = O(\sqrt{d}\log n)$.
    Combining this with \cref{eq:ds-upper-corr}, we have
    \[
    \gamma \cdot n = O(\sqrt{d}\log n),
    \]
    which solves to $n = O((\sqrt{d}/\gamma)\log(d/\gamma)).$ This concludes the proof.
\end{proof}

\subsection{Lower Bounds}
In this section, we provide two lower bounds on the $\gamma$-VC dimension of decision stumps. 

\subsubsection{A looser lower bound}
\begin{theorem}[\cref{thm:stumps-lower} restated, $d \ge 1/\gamma^2$]
    \label{thm:ds-lower_bound_loose}
    For every integer $d \ge 1$ and $\gamma \in (0, 1]$, if $d \ge 1/\gamma$, then $\VC_\gamma(\DS_d) = \Omega(\log(\gamma d)/\gamma)$.
\end{theorem}

The lower bound above is different from the second term of \cref{thm:stumps-lower}, which states that $\VC_\gamma(\DS_d) = \Omega(\max\{\log(\gamma d)/\gamma, 1/\gamma^{2}\})$ for $d \ge 1/\gamma^2$.
This is because when $1/\gamma \le d \le 1/\gamma^2$, we have $\sqrt{d}/\gamma \ge \sqrt{1/\gamma}/\gamma \ge \log(1/\gamma)/\gamma \ge \log(\gamma d)/\gamma$, which means the first term of \cref{thm:stumps-lower} dominates.
When $d \ge 1/\gamma^2$, we always have $\VC_\gamma(\DS_d) \ge \VC_\gamma(DS_{\lfloor\gamma^{-2}\rfloor}) = \Omega(1/\gamma^2)$, where the last equality follows from the first term of \cref{thm:stumps-lower}.

\begin{proof}[Proof of \cref{thm:ds-lower_bound_loose}]
    \textbf{Case 1.} If $\gamma = 1$, then $\VC_\gamma(\DS_d) = \VC(\DS_d)$, and the latter is known to be $\Theta(\log d)$ as shown in~\citet{shalev2014understanding}.
    Thus, the lower bound holds.

    \textbf{Case 2.} If $1/2 < \gamma < 1$, then $d \ge 2$.
    We have $\VC_\gamma(\DS_d) \ge \VC(\DS_d) = \Theta(\log d)$, and it suffices to show that $\log d \ge \log(2\gamma d)/(4\gamma)$, which is equivalent to $d^{4\gamma-1} \ge 2\gamma$.
    Since $d \ge 2$ and $1/2 < \gamma < 1$, we have $d^{4\gamma-1} > 2 > 2\gamma$, which concludes the proof.

    \textbf{Case 3.} If $0 < \gamma \le 1/2$, then $d \ge 2$ and $1/(2\gamma) \ge 1$.
    Since rounding $\gamma$ up to the nearest power of $1/2$ (which is at most $1/2$) only changes the bound by a constant factor, we assume without loss of generality that $\gamma$ is a power of $1/2$.
    We now construct a sequence of $n = \lfloor\log(2\gamma d)\rfloor/(2\gamma)$ points in $\bbR^d$ that is $\gamma$-interpolated by $\DS_d$.

    Let $P = (\bp_1, \dotsc, \bp_n)$ be a sequence of $n$ points in $\bbR^d$.
    Split $P$ into $k = 1/(2\gamma)$ subsequences $P_1, \dotsc, P_k$ of equal length, then the length of each subsequence is $l = 2\gamma n$.
    Let $A \in \bbR^{d \times n}$ be the matrix whose $j$-th column is $\bp_j$.
    Since $d \ge 2^{l} \cdot k$, we can define the coordinates of the points in $P$ as follows:
    \begin{itemize}
        \item For the $i$-th subsequence $P_i$, let $A_i$ be the submatrix of $A$ restricted to columns $[1 + l(i - 1), li]$ and rows $[1 + 2^l(i - 1), 2^l i]$.
        We set each row of $A_i$ to a distinct $\{\pm 1\}$-pattern of length $l$.
        \item All other entries of $A$ are set to $0$.
    \end{itemize}

    \begin{figure}[htbp]
        \centering
        \includegraphics[width=0.9\textwidth]{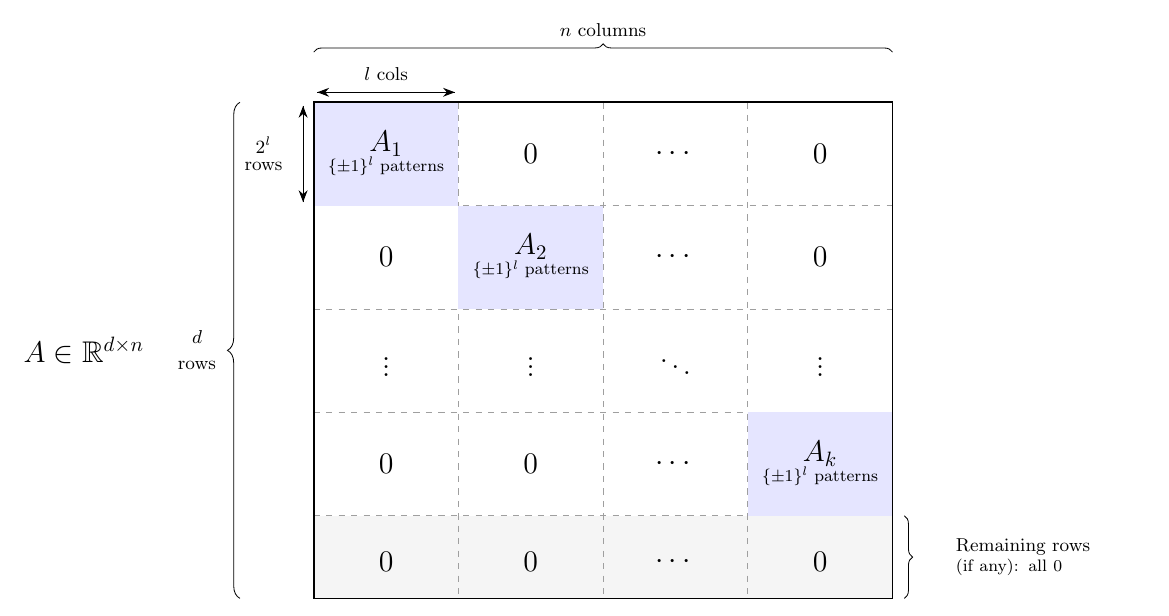}
        \caption{Illustration of the block construction.
        Each submatrix $A_i$ corresponds to a subsequence $P_i$, and the rows of $A_i$ are distinct $\{\pm 1\}$-patterns of length $l$.}
        \label{fig:ds-lower-loose}
    \end{figure}

    Next, we show that such a point sequence $P$ is $\gamma$-interpolated by $\DS_d$.
    Let $c\colon P \to \{\pm 1\}$ be any concept, and let $\cD$ be any distribution over $P$.
    We define a weight vector $\bw \in \bbR^n:\norm{\bw}_1 = 1$ such that for any point $\bp \in P$, $w_\bp = \cD(\bp) \cdot c(\bp)$.
    Our goal is to show that there exists a decision stump $h$ such that $\bbE_{\bp \sim \cD}[c(\bp)h(\bp)] = \bw \cdot \bh \ge \gamma$, where $\bh\in \{\pm 1\}^n$ is the projection of $h$ on $P$.

    Since there are $k$ subsequences, there must exist a subsequence $P_i$ such that $\cD(P_i) = \sum_{\bp \in P_i} \abs{w_\bp} \ge 1/k$. Fix such a subsequence $P_i$.
    Let $\bw_i$ be the subvector of $\bw$ corresponding to the points in $P_i$.
    By the construction of the matrix $A$, we can find a row $\br=(r_1,\ldots,r_n)\in\{-1, 0, 1\}^n$ of $A$, whose projection on $P_i$ is exactly the sign pattern of $\bw_i$, and the projection on all other points is $\mathbf{0}$.
    Then, we have
    \[
    \bw \cdot \br = \sum_{j=1}^{n} w_j r_j = \sum_{\bp \in P_i} w_\bp r_\bp = \sum_{\bp \in P_i} w_\bp \cdot \sign(w_\bp) = \cD(P_i) \ge \frac{1}{k} = 2\gamma,
    \]
    where the last equality holds because $k = 1/(2\gamma)$.
    Finally, we show that there exists a convex combination of four decision stumps that approximates the vector $\br$.
    Suppose the row $\br$ corresponds to the $r$-th feature, we define a function $g$ as follows:
    \[
    g(\bp) = \frac{\ind^{\pm}\{p_r \ge 1/2\} + \ind^{\pm}\{p_r < 3/2\} + \ind^{\pm}\{p_r < -3/2\} + \ind^{\pm}\{p_r \ge -1/2\}}{4},
    \]
    where $\ind^{\pm}\{\cdot\}$ is the indicator function that outputs $1$ if the condition is satisfied and $-1$ otherwise.
    And we have
    \[
    g(\bp) = \begin{cases}
        \frac{1}{2} & \text{if } p_r = 1 \\
        0 & \text{if } p_r = 0 \\
        -\frac{1}{2} & \text{if } p_r = -1
    \end{cases}.
    \]
    Therefore, let $\bg\in \{-\frac{1}{2}, 0, \frac{1}{2}\}^n$ be the projection of $g$ on $P$, we have $\bg = 1/2 \cdot \br$, which implies that $\bw \cdot \bg = 1/2 \cdot \bw \cdot \br \ge \gamma$.
    Besides, since $g$ is a convex combination of four decision stumps, there must exist a decision stump $h$ such that $\bw \cdot \bh \ge \bw \cdot \bg \ge \gamma$.
    This concludes the proof.

\end{proof}

\subsubsection{A tighter lower bound}
\begin{theorem}[\cref{thm:stumps-lower} restated, $d \le 1/\gamma^2$]
    \label{thm:ds-lower_bound_tight}
    For every integer $d \ge 1$ and $\gamma \in (0, 1]$, if $d \le 1/\gamma^2$, then $\VC_\gamma(\DS_d) = \Omega(\sqrt{d}/\gamma)$.
\end{theorem}
\begin{proof}[Proof of \cref{thm:ds-lower_bound_tight}]
    \textbf{Case 1.} If $1/4 < \gamma \le 1$, then $\sqrt{d}/(8\gamma) \le 1/(8\gamma^2) < 2$.
    By the definition of $\gamma$-VC dimension, we have $\VC_\gamma(\DS_d) \ge \VC(\DS_d) \ge 2 > \sqrt{d}/(8\gamma)$, which concludes the proof.

    \textbf{Case 2.} If $0 < \gamma \le 1/4$, then $1/(16\gamma^2) \ge 1$.
    Since rounding $d$ down to the largest power of $4$ not exceeding $\min\{d, 1/(16\gamma^2)\}$ (which is at least $1$) only changes the bound by a constant factor, we assume that $d$ is a power of $4$ and $d \le 1/(16\gamma^2)$ from now on.
    Define an integer $k$ to be $\lfloor 1/(2\gamma \sqrt{d})\rfloor$.
    Since $1 \le d \le 1/(16\gamma^2)$, we have
    \[
    k = \lfloor\frac{1}{2\gamma \sqrt{d}}\rfloor
    > \frac{1}{2\gamma \sqrt{d}} - 1
    \ge \frac{1}{2\gamma \sqrt{d}} - \frac{1}{4\gamma \sqrt{d}}
    = \frac{1}{4\gamma \sqrt{d}}
    \ge 1.
    \]

    We now construct a sequence of $n = d \cdot k > \sqrt{d}/(4\gamma)$ points in $\bbR^d$ that is $\gamma$-interpolated by $\DS_d$.
    Let $P = (\bp_1, \dotsc, \bp_n)$ be a sequence of $n$ points in $\bbR^d$.
    Split $P$ into $k$ subsequences $P_1, \dotsc, P_k$ of equal length $d$.
    Let $A \in \bbR^{d \times n}$ be the matrix whose $j$-th column is $\bp_j$, and let $A_i$ be the $d \times d$ submatrix of $A$ restricted to columns $[1 + d(i - 1), di]$, which corresponds to the subsequence $P_i$.
    The values of the entries of $A$ are defined as follows: for each $i \in [k]$, we set $A_i = i \cdot H_d$ where $H_d$ is a Hadamard matrix of order $d$.

    \begin{figure}[htbp]
        \centering
        \includegraphics[width=0.6\textwidth]{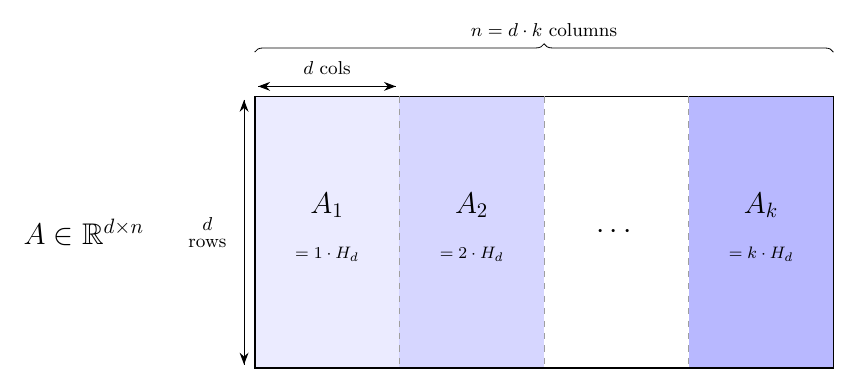}
        \caption{Illustration of the block construction.
        Each submatrix $A_i$ is a $i$-scaled copy of a Hadamard matrix, and corresponds to the subsequence $P_i$.}
        \label{fig:ds-lower-tight}
    \end{figure}

    Next, we show that such a point sequence $P$ is $\gamma$-interpolated by $\DS_d$.
    Let $c\colon P \to \{\pm 1\}$ be any concept, and let $\cD$ be any distribution over $P$.
    We define a weight vector $\bw \in \bbR^n:\norm{\bw}_1 = 1$ such that for any point $\bp \in P$, $w_\bp = \cD(\bp) \cdot c(\bp)$.
    Our goal is to show that there exists a decision stump $h$ such that $\bbE_{\bp \sim \cD}[c(\bp)h(\bp)] = \bw \cdot \bh \ge \gamma$, where $\bh\in \{\pm 1\}^n$ is the projection of $h$ on $P$.

    Since there are $k$ subsequences, there must exist a subsequence $P_i$ such that $\cD(P_i) = \sum_{\bp \in P_i} \abs{w_\bp} \ge 1/k \ge 2\gamma \sqrt{d}$. Fix such a subsequence $P_i$.
    Let $\bw_i$ be the subvector of $\bw$ corresponding to the points in $P_i$.
    Since $H_d$ is a Hadamard matrix, for any vector $\bx$, we have $\norm{H_d\bx}_2 = \sqrt{d} \cdot \norm{\bx}_2$.
    Therefore, we have
    \begin{align*}
    \norm{H_d\bw_i}_\infty
    \ge \frac{\norm{H_d\bw_i}_2}{\sqrt{d}}
    = \norm{\bw_i}_2
    \ge \frac{\norm{\bw_i}_1}{\sqrt{d}}
    = \frac{\cD(P_i)}{\sqrt{d}}
    \ge 2\gamma,
    \end{align*}
    where the second inequality comes from Cauchy-Schwarz. This implies that there exists a row $\br_i\in \{\pm 1\}^d$ of $H_d$ such that $\abs{\bw_i \cdot \br_i} \ge 2\gamma$.
    Let vector $\br\in\{-1, 0, 1\}^n$ be the extension of $\br_i$ to length $n$, i.e., the projection of $\br$ on $P_i$ is exactly $\br_i$, and the projection on all other points is $0$.
    Then, we have $\abs{\bw \cdot \br} = \abs{\bw_i \cdot \br_i} \ge 2\gamma$.

    Finally, we show that there exists a convex combination of four decision stumps that approximates the vector $\br$.
    Suppose the row $\br_i$ corresponds to the $r$-th feature, we define a function $g$ as follows:
    \[
    g(\bp) = \frac{\ind^{\pm}\{p_r \ge i - \frac{1}{2}\} + \ind^{\pm}\{p_r < i + \frac{1}{2}\} + \ind^{\pm}\{p_r < -i - \frac{1}{2}\} + \ind^{\pm}\{p_r \ge -i + \frac{1}{2}\}}{4},
    \]
    where $\ind^{\pm}\{\cdot\}$ is the indicator function that outputs $1$ if the condition is satisfied and $-1$ otherwise.
    And we have
    \[
    g(\bp) = \begin{cases}
        \frac{1}{2} & \text{if } p_r = i \\
        0 & \text{if } p_r \in \{-k, \dotsc, -1, 1, \dotsc, k\} \setminus \{i, -i\} \\
        -\frac{1}{2} & \text{if } p_r = -i
    \end{cases}.
    \]
    Therefore, let $\bg\in \{-\frac{1}{2}, 0, \frac{1}{2}\}^n$ be the projection of $g$ on $P$, we have $\bg = 1/2 \cdot \br$.
    By the symmetry of decision stumps, we can construct $g'$ such that $\bg' = -\bg = -1/2 \cdot \br$.
    Therefore, we have $\bw \cdot \bg = 1/2 \cdot \bw \cdot \br \ge \gamma$ or $\bw \cdot \bg' = -\bw \cdot \bg = -1/2 \cdot \bw \cdot \br \ge \gamma$.
    Since both $g$ and $g'$ are convex combinations of four decision stumps, there exists a decision stump $h$ such that $\bw \cdot \bh \ge \max\{\bw \cdot \bg, \bw \cdot \bg'\} \ge \gamma$.
    This concludes the proof.
\end{proof}

\section{Axis-Parallel Rectangles}
\label{sec:axisparallel}

In this section, we analyze the $\gamma$-VC dimension of the class of high-dimensional axis-parallel rectangles.
Let $\cR_d$ be the set of all axis-parallel rectangles in $\bbR^d$. The class of axis-parallel rectangles contains a corresponding hypothesis $h_R$ for any axis-parallel rectangle $R \in \cR_d$, where $h_R(x) = 1$ if $x \in R$ and $-1$ otherwise. Notably, \citet[Section 5]{alon2023boosting} assumed that the base hypothesis class is symmetric, i.e., if $h \in \cH$, then $-h \in \cH$.
To make the class symmetric, for any axis-parallel rectangle $R \in \cR_d$, we include two hypotheses $h_R$ and $-h_R$ in the class, where $h_R(\bx) = 1$ if $\bx \in R$ and $-1$ otherwise.

Formally, we denote the class of $d$-dimensional axis-parallel rectangles by $\R_d$, defined as follows:
\[
\R_d = \left\{h_R, -h_R: R = \prod_{i=1}^d [a_i, b_i], a_i \le b_i, a_i, b_i \in \bbR\right\},
\]
where for any point $\bx = (x_1, \dotsc, x_d) \in \bbR^d$,
\[
h_R(\bx) = \begin{cases}
1 & \text{if } a_i \le x_i \le b_i \text{ for all } i\in [d], \\
-1 & \text{otherwise}.
\end{cases}
\]

Bounding the $\gamma$-VC dimension of $\R_d$ is highly related to high-dimensional Tusnády's problem that we briefly introduce in the following. We refer readers to~\citet{matouvsek2020factorization} for a more comprehensive survey on this problem.

\paragraph{Tusnády's problem.}
Let $P$ be a set of $n$ points in $\bbR^2$ and $\cR_2$ be the set of 2-dimensional axis-parallel rectangles.
Define $\cR_2(P) \coloneqq \{R\cap P: R \in \cR_2\}$ as the set of all subsets of $P$ that can be obtained by intersecting $P$ with an axis-parallel rectangle.
The goal is to understand the discrepancy of $\cR_2(P)$.
The high-dimensional version of Tusnády's problem is defined similarly, where $\cR_2$ is replaced by $\cR_d$.

For the best known discrepancy upper bound, \citet[Theorem 1]{nikolov2017tighter} provides an upper bound on the discrepancy of the set systems produced by intersecting a point set with $d$-dimensional anchored rectangles. Substituting this result into the analysis of \citet[Theorem 1.1]{matouvsek2020factorization} immediately gives an upper bound of $O_d(\log^{d-1/2}n)$.
Based on this, we prove that the $\gamma$-VC dimension of $\R_d$ is upper bounded by $O_d\left(\log^{d-1/2}(1/\gamma)/\gamma\right)$.

For the lower bound, \citet[Theorem 1.1]{matouvsek2020factorization} gives a lower bound of $\Omega_d(\log^{d-1}n)$, which is tight up to a factor of $\sqrt{\log n}$.

Unfortunately it is not clear how to translate this lower bound to the $\gamma$-VC dimension setting. In particular, discrepancy naturally corresponds to the \emph{uniform distribution} being $\gamma$-realizable by $\cH$. This is not a problem for upper bounding $\gamma$-VC dimension in terms of discrepancy upper bounds as $\gamma$-interpolation requires \emph{every} distribution to be $\gamma$-realizable. For lower bounds, discrepancy lower bounds do not necessarily imply that every distribution (in particular highly non-uniform distributions) is $\gamma$-realizable by $\cH$.

To get a meaningful lower bound for the $\gamma$-VC dimension, we instead adapt a discretized variant of Roth's method (see~\citet{roth1954irregularities,chazelle1997lower}) and generalize it to higher dimensions.
Roth's method is used to prove a lower bound for geometric discrepancy of axis-parallel rectangles, and we refer readers to \citet[Section 6.1]{matousek1999geometric} and \citet[Section 3.1]{DBLP:books/daglib/0012928} for textbook treatments.

\subsection{Upper Bound}
Substituting the result of \citet[Theorem 1]{nikolov2017tighter} into the analysis of \citet[Theorem 1.1]{matouvsek2020factorization} gives us the following upper bound for the discrepancy of $\cR_d(P)$.
\begin{lemma}\label{lem:rectangles-bound-discrepancy}
    For every fixed integer $d \ge 2$, every set $P$ of $n$ points in $\bbR^d$, $\disc(\cR_d(P)) = O_d(\log^{d-1/2}n)$.
\end{lemma}

With this in place, we first bound the signed discrepancy and then prove the following upper bound for the $\gamma$-VC dimension of $\R_d$.

\begin{theorem}[\cref{thm:intro-rectangles-upper} restated]
    \label{thm:rectangles-upper-bound}
    For every fixed integer $d \ge 2$ and $\gamma \in (0, 1]$, $\VC_\gamma(\R_d) = O_d(\log^{d-1/2}(1/\gamma)/\gamma)$.
\end{theorem}

\begin{proof}[Proof of \cref{thm:rectangles-upper-bound}]
    Let $P$ be a set of $n$ points in $\bbR^d$ that is $\gamma$-interpolated by $\R_d$.
    Let $\cD$ be the uniform distribution over $P$.
    Then, for any concept $c\colon P \to \{\pm 1\}$, there exists a rectangle $R$ such that
    \[
    \abs{\corr_\cD(h_R)}
    = \abs{\corr_\cD(-h_R)}
    = \frac{1}{n} \cdot \abs*{\sum_{p \in R} c(p) - \sum_{p \notin R} c(p)}
    \ge \gamma.
    \]

    On the other hand, since there is always a rectangle that can contain all points in $P$, \cref{lem:signed-discrepancy-bound} and \cref{lem:rectangles-bound-discrepancy} upper bounds the signed discrepancy of $\cR_d(P)$, i.e., $\discpm(\cR_d(P)) \le 3\disc(\cR_d(P)) = O_d(\log^{d-1/2}n)$.
    This implies that there exists a concept $c_1\colon P \to \{\pm 1\}$ such that for any rectangle $R$, $\abs*{\sum_{p \in R} c_1(p) - \sum_{p \notin R} c_1(p)} = O_d(\log^{d-1/2}n)$.
    Combining with the previous inequality, we have
    \[
    \gamma \cdot n = O_d(\log^{d-1/2}n),
    \]
    which solves to $n = O_d(\log^{d-1/2}(1/\gamma)/\gamma)$.
\end{proof}

\subsection{Lower Bound}

In this section, we prove a lower bound for the $\gamma$-VC dimension of $\R_d$.
Specifically, we adapt a discretized variant of Roth's method~\citep{roth1954irregularities,chazelle1997lower} and generalize it to higher dimensions. Briefly speaking, we construct a set $P$ of $n$ points in $\bbR^d$ such that for any coloring $c\colon P \to \{\pm 1\}$, and any distribution $\cD$ over $P$, there exists a rectangle $R$ such that $\corr_\cD(h_R) = \Omega_d(\log^{(d-1)/2}n/n)$, which implies $\VC_\gamma(\R_d) = \Omega_d(\log^{(d-1)/2}(1/\gamma)/\gamma)$.

Fix integers $n\ge 1$ and $d\ge 2$. Let $m$ be a power of 2, such that $n \le m/(2^{d+1}e) < 2n$.
Given the unit box $[0, 1]^d$, we partition it into $m$ axis-parallel $d$-dimensional cells of same volume $1/m$.
Formally, let $\bk = (k_1, \dotsc, k_d)\in \bbN^d$ be a vector which satisfies $0 \le k_i \le \log m$ and $\sum_i k_i = \log m$. For any $\bk$, the unit box is partitioned as follows: on the $i$-th dimension, we partition the interval $[0, 1]$ into $2^{k_i}$ sub-intervals of the same length $2^{-k_i}$, and we take the Cartesian product of these sub-intervals across all dimensions to get these cells.
And we call the resulting grid $\cG_\bk$.

\begin{figure}[htbp]
    \centering
    \includegraphics[width=0.9\textwidth]{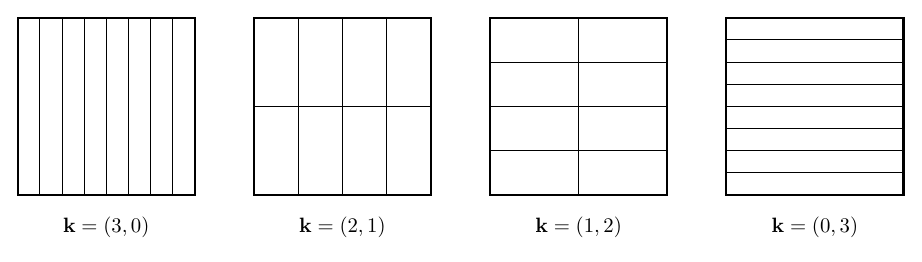}
    \caption{Illustration of different grid constructions when $d = 2$ and $m = 8$.}
\end{figure}

Let $K$ be the set of all vectors $\bk$ satisfying the above conditions.
Using stars and bars\footnote{See \url{https://en.wikipedia.org/wiki/Stars_and_bars_(combinatorics)}.},  we can calculate the size of $K$ as follows:
\[
\abs{K} = \binom{\log m + d - 1}{d - 1} \ge \frac{\log^{d-1}m}{(d - 1)!} = \Omega_d(\log^{d-1}m).
\]

A point $\bp = (p_1, \dotsc, p_d)$ is said to be \emph{well-centered} for a grid $\cG_\bk$ if $\bp$ is contained in the center sub-cell of a cell $\sigma= \prod_{i=1}^d [a_i, b_i]$ of $\cG_\bk$, i.e., $a_i + (b_i - a_i)/4 < p_i < a_i + 3(b_i - a_i)/4$ for all $i\in [d]$.
Given a point set and a grid, we say a cell of the grid is \emph{good} if it only contains one point and this point is in the center sub-cell of the cell.

\begin{lemma}[Existence of a good point set]\label{lem:rectangles-good-point-set}
    There exists a point set of size at least $m/(2^{d+1}e)$ in $[0, 1]^d$ such that each point is contained in a good cell for at least $\abs{K}/(2^{d+1}e)$ grids.\footnote{The corresponding grids for different points can be different.}
\end{lemma}
\begin{proof}[Proof of \cref{lem:rectangles-good-point-set}]
    We use the probabilistic method to prove the existence of such a point set.
    Let $\{\bp_1, \dotsc, \bp_m\}$ be a set of $m$ random points that are sampled independently and uniformly from $[0, 1]^d$.

    Fix any grid $\cG_\bk$.
    For any cell $\sigma$ of $\cG_\bk$, define an indicator random variable $X_\sigma$ that equals 1 if $\sigma$ is good, and equals 0 otherwise.
    Then, the probability of the event $X_\sigma = 1$ is
    \begin{align*}
        \bbP_{P}[X_\sigma = 1]
        &= \sum_{i=1}^m \left(\bbP[\bp_i\text{ is in the center sub-cell of } \sigma] \cdot \prod_{j=1:j \neq i}^{m} \bbP[\bp_j \notin \sigma]\right)\\
        &= m \cdot \frac{1}{2^d \cdot m} \cdot \left(1 - \frac{1}{m}\right)^{m-1}
        \ge \frac{1}{2^de},
    \end{align*}
    where the $1/(2^d \cdot m)$ term is exactly the volume of the center sub-cell of $\sigma$.
    Let $X_\bk = \sum_{\sigma \in \cG_\bk} X_\sigma$ be the number of good cells in $\cG_\bk$.
    Then, we have
    \[
    \bbE_{P}[X_\bk]
    = \sum_{\sigma \in \cG_\bk} \bbE_{P}[X_\sigma]
    = \sum_{\sigma \in \cG_\bk} \bbP_{P}[X_\sigma = 1]
    \ge \frac{m}{2^de}.
    \]
    Summing over all $\bk \in K$, we have
    \[
    \bbE_{P}[\sum_{\bk \in K} X_\bk]
    = \sum_{\bk \in K} \bbE_{P}[X_\bk]
    \ge \frac{m\abs{K}}{2^de},
    \]
    which implies that there exists a point set $P$ of size $m$ such that there are at least $m\abs{K}/(2^de)$ good cells across all grids.

    With this point set $P$, we now prove the lemma by contradiction.
    Suppose the lemma doesn't hold, then there are at most $m/(2^{d+1}e)-1$ points that are contained in a good cell for at least $\abs{K}/(2^{d+1}e)$ grids.
    Therefore, the total number of good cells across all grids is less than 
    \[
    m/(2^{d+1}e) \cdot \abs{K} + (m - m/(2^{d+1}e)) \cdot \abs{K}/(2^{d+1}e) < m\abs{K}/(2^d \cdot e) ,
    \]
    which contradicts the previous conclusion that there are at least $m\abs{K}/(2^de)$ good cells across all grids.
    Therefore, there are at least $m/(2^{d+1}e)$ points in $P$ that each is contained in a good cell for at least $\abs{K}/(2^{d+1}e)$ grids, which concludes the proof.

\end{proof}

Fix a point set $P$ of size $n$ in $[0, 1]^d$ that is guaranteed by \cref{lem:rectangles-good-point-set}. We now show that for any coloring $c\colon P \to \{\pm 1\}$ and any distribution $\cD$ over $P$, there exists a rectangle $R$ such that $\abs{\corr_\cD(h_R)} = \Omega_d(\log^{(d-1)/2}n/n)$.

Consider a square grid $\cG$ that in each dimension divide $[0, 1]$ into $m^2$ sub-intervals of the same length $1/m^2$.
Then $\cG$ has $m^{2d}$ cells and $N = (m^2 + 1)^d$ grid points.
For the point set $P$, we construct a $N \times n$ matrix $B$ as follows: the rows of $B$ are indexed by the grid points of $\cG$, and the columns of $B$ are indexed by the points in $P$.
For any grid point $\bq = (q_1, \dotsc, q_d)$ and any point $\bp = (p_1, \dotsc, p_d) \in P$, we set $B(\bq, \bp) = 1$ if $\bp\le \bq$, i.e., $p_i \le q_i$ for all $i\in [d]$, and set $B(\bq, \bp) = 0$ otherwise.

Moreover, for a coloring $c\colon P \to \{\pm 1\}$, and a distribution $\cD$ over $P$, we define a weight vector $\bw \in \bbR^n:\norm{\bw}_1 = 1$ such that for any point $\bp \in P$, $w_\bp = \cD(\bp) \cdot c(\bp)$.
Consider the vector $B\bw \in \bbR^N$, each entry corresponds to a grid point $\bq$ and equals $\sum_{\bp \in P: \bp \le \bq} w_\bp$.

\begin{lemma}\label{lem:rectangles-large-infinity-norm}
    Let $P$ be a point set of size $n$ in $[0, 1]^d$, for which each point is contained in a good cell for at least $\abs{K}/(2^{d+1}e)$ grids.
    For every weight vector $\bw \in \bbR^n$ with $\norm{\bw}_1 = 1$ over the points in $P$, $\norm{B\bw}_\infty = \Omega_d(\log^{(d-1)/2} n/n)$.
\end{lemma}
\begin{proof}[Proof of \cref{lem:rectangles-large-infinity-norm}]
    Such a point set $P$ exists by \cref{lem:rectangles-good-point-set} and that $n \le m/(2^{d+1}e)$.
    Because $\norm{B(-\bw)}_\infty = \norm{B\bw}_\infty$, we can assume that
    \begin{equation}\label{eq:rectangles-assumption-positive-weight}
    \sum_{w_\bp>0} w_\bp \ge \frac{1}{2}.
    \end{equation}

    For any $\bk \in K$, we say that a cell $\sigma$ of $\cG_\bk$ is \emph{positively good} if $\sigma$ is good and the point contained in $\sigma$ has a positive weight, i.e., $\sigma$ contains exactly one point $\bp \in P$, which is contained in the center sub-cell of $\sigma$ and $w_\bp > 0$.
    Next, for each grid $\cG_\bk$, we assign a value $g_{\bk, \bq} $ to each grid point $\bq$ of $\cG$, and collect these values into a vector $\bg_\bk \in \bbR^N$, in which the points are ordered in the same order as the rows of $B$.
    The value $g_{\bk, \bq} $ is defined as follows:
    \begin{itemize}
        \item If $\bq$ lies on the boundary of more than one cells or $\bq$ is not contained in any positively good cell, then we set $g_{\bk, \bq} = 0$.
        \item Else, $\bq$ lies in a positively good cell $\sigma$ of $\cG_\bk$, then we divide the cell into $2^d$ sub-cells by cutting $\sigma$ in the middle of each dimension,
        \begin{itemize}
            \item If $\bq$ lies on the boundary of these sub-cells, then we set $g_{\bk, \bq} = 0$.
            \item Else, $\bq$ lies in the interior of a sub-cell $\sigma'$, let the center point of $\sigma$ be $\bq_{\sigma}$ and the center point of $\sigma'$ be $\bq_{\sigma'}$.
            Then, if $\sum_{i=1}^{d} \ind\{q_{\sigma', i} < q_{\sigma, i}\}$ is even, we set $g_{\bk, \bq} = 1$, otherwise we set $g_{\bk, \bq} = -1$.
        \end{itemize}
    \end{itemize}

    An illustration of such a construction is shown in \cref{fig:rec-value-function}.
    We can verify that for any two different grids $\cG_{\bk_1}$ and $\cG_{\bk_2}$, the corresponding value vectors $\bg_{\bk_1}$ and $\bg_{\bk_2}$ are orthogonal, i.e., $\bg_{\bk_1}^\top \bg_{\bk_2} = 0$, see \cref{fig:rec-orthogonality} for an illustration.
    Recall that $\sum_i k_{1,i} = \sum_i k_{2,i} = \log m$, so since $\bk_1 \neq \bk_2$ there is a dimension $i$ with $k_{2,i} > k_{1,i}$. 
    In this dimension a sub-cell $\sigma_1'$ of a (positively good) cell of $\cG_{\bk_1}$ has length $2^{-k_{1,i}-1} \ge 2^{-k_{2,i}}$, an integer multiple of the $\cG_{\bk_2}$-cell
    length, which means $\sigma_1'$ is a union of some whole $\cG_{\bk_2}$-cells in dimension $i$. Fix such a $\sigma_1'$, on which $g_{\bk_1}$ is a constant value, and a $\cG_{\bk_2}$-cell $\sigma_2$ intersecting it. 
    If $\sigma_2$ is not positively good then $g_{\bk_2} \equiv 0$ on $\sigma_1' \cap \sigma_2$; otherwise, $\sigma_2$ is positively good and $g_{\bk_2}$ takes values $1$ and $-1$ on the two halves of $\sigma_2$ in dimension $i$, which are both contained in $\sigma_1'$.
    Hence, these grid points contribute $0$ to $\bg_{\bk_1}^\top \bg_{\bk_2}$. 
    Summing over all sub-cells $\sigma_1'$ of cells of $\cG_{\bk_1}$ and all cells $\sigma_2$ of $\cG_{\bk_2}$ gives $\bg_{\bk_1}^\top \bg_{\bk_2} = 0$.
    
    Let $G$ denote a $N \times |K|$ matrix whose columns are the vectors $\bg_\bk$ for $\bk \in K$ and let $\bu$ be the all-ones vector of length $|K|$, the orthogonality property implies that
    \begin{align}
    \norm{G\bu}_2^2
    = \sum_{\bk \in K} \norm{\bg_\bk}_2^2
    \le \sum_{\bk \in K} N \cdot 1^2
    = N \cdot |K|. \label{eq:rectangles-Gu-upper-bound}
    \end{align}

    \begin{figure}[htbp]
        \centering

        \begin{subfigure}{0.48\textwidth}
            \centering
            \includegraphics[width=\textwidth]{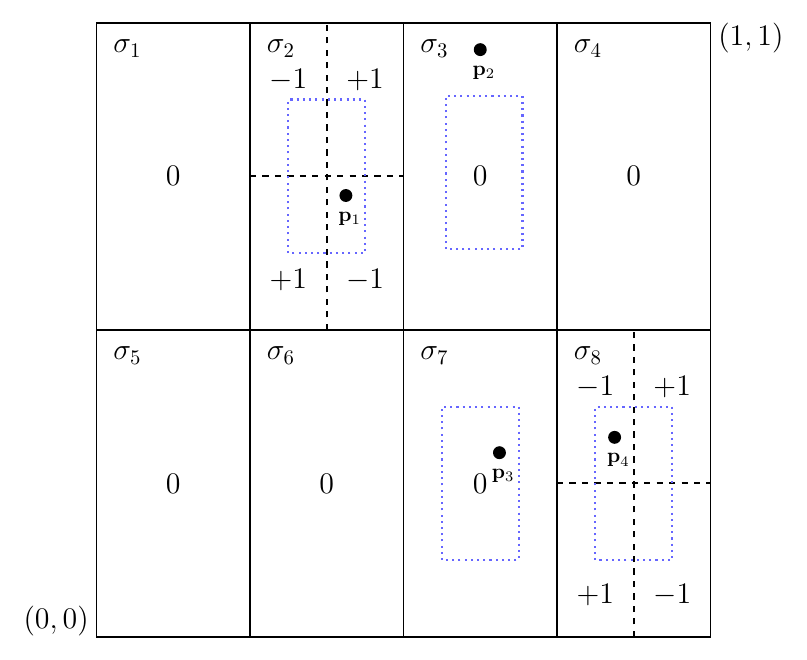}
            \caption{Illustration of a value function $g_\bk$ in the 2D grid $\cG_\bk$ with $\bk = (2, 1)$.
            There are 4 points $\bp_1, \bp_2, \bp_3, \bp_4$ with $w_{\bp_1}, w_{\bp_2}, w_{\bp_4} > 0$ and $w_{\bp_3} < 0$.
            The blue boxes show the center sub-cells. $\sigma_2, \sigma_8$ are positively good cells, and are split into four sub-cells to assign different values. }
            \label{fig:rec-value-function}
        \end{subfigure}
        \hfill
        \begin{subfigure}{0.48\textwidth}
            \centering
            \includegraphics[width=\textwidth]{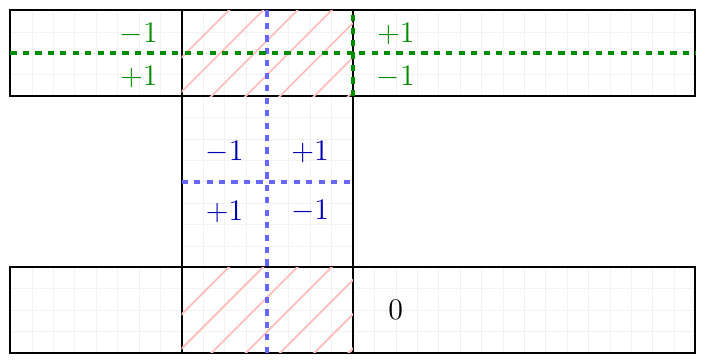}
            \caption{Illustration of the orthogonality property.
            The vertical rectangle represents a positively good cell $\sigma$ of $\cG_{\bk_1}$ where $\bk_1 = (2, 1)$.
            The two horizontal rectangles represent two cells of $\cG_{\bk_2}$ where $\bk_2 = (0, 3)$, one of which is a positively good cell.
            The grid points in the overlap of these cells have contribution 0 to $\bg_{\bk_1}^\top \bg_{\bk_2}$.}
            \label{fig:rec-orthogonality}
        \end{subfigure}

        \caption{Illustrations of the construction.
        The numbers $\{-1, 0, +1\}$ in a cell or sub-cell represent the value of $g_\bk$ at the grid points in that cell or sub-cell.}
        \label{fig:two-figures}
    \end{figure}

    We now consider the product $\bg_\bk^\top B\bw$, which can be written as
    \[
    \bg_\bk^\top B\bw
    = \sum_{\bq} \left(g_{\bk, \bq} \cdot \sum_{\bp \in P: \bp \le \bq} w_\bp\right)
    = \sum_{\bp \in P} \left(w_\bp \cdot \sum_{\bq: \bq \ge \bp} g_{\bk, \bq}\right),
    \]
    By the construction, if $\bp$ is not contained in any positively good cell, then $\sum_{\bq: \bq \ge \bp} g_{\bk, \bq} = 0$.
    Else, if $\bp$ is contained in a positively good cell $\sigma$, the sum of contribution from the grid points outside $\sigma$ is also 0, since the number of grid points outside $\sigma$ that satisfy $\bq \ge \bp$ and have value 1 is the same as the number of grid points outside $\sigma$ that satisfy $\bq \ge \bp$ and have value -1.
    What remains is the contribution from the grid points inside $\sigma$.

    Let $\sigma$ be a positively good cell of $\cG_\bk$ that contains $\bp$, and let $\bq_{\sigma}$ be its center point.
    For each grid point $\bq$ in the interior of $\sigma$, we group it with its $2^d - 1$ reflections (one in each sub-cell, and each reflection $\bq'$ satisfies that on every dimension, the distance from $\bq'$ to $\bq_{\sigma}$ is the same as the distance from $\bq$ to $\bq_{\sigma}$).
    For such a group of $2^d$ grid points, the contribution to $\sum_{\bq: \bq \ge \bp} g_{\bk, \bq}$ is 1 if and only if $\bp$ lies in the $d$-dimensional box defined by these $2^d$ grid points, otherwise the contribution is 0.

    Let $\sigma'$ be the sub-cell of $\sigma$ whose center point is closest to the origin, and divide $\sigma'$ into $2^d$ sub-cells by cutting $\sigma'$ in the middle of each dimension, and let $\sigma''$ be the sub-cell whose center point is closest to the origin.
    Then, the box defined by any grid point $\bq\in \sigma''$ and its reflections contains $\bp$. Thus the contribution from this point group is 1.

    \begin{figure}[htbp]
        \centering
        \includegraphics[width=0.5\textwidth]{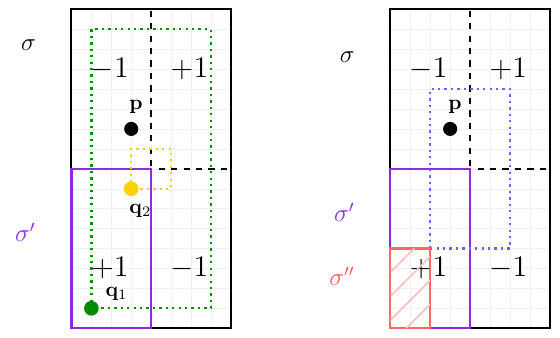}
        \caption{Illustration of the contribution from the grid points inside a positively good cell $\sigma$ to $\sum_{\bq: \bq \ge \bp} g_{\bk, \bq}$ for a point $\bp$ contained in $\sigma$.
        In the left figure, the green (yellow) box represents the box formed by $\bq_1 (\bq_2)$ and its reflections.
        The contribution from the 4 grid points which form the green (yellow) boxes is 1 (0).
        In the right figure, the blue box represents the center sub-cell of $\sigma$.
        For every grid point in the red box, the contribution from it and its reflections is 1.}
    \end{figure}

    Notice that the volume of $\sigma''$ is $1/(2^{2d}m)$, and these grid points are uniformly distributed with distance $1/m^2$ on each dimension.
    Thus, the number of grid points in $\sigma''$ with non-zero value (excluding the grid points on the boundary of $\sigma$) is $m^{2d}/(2^{2d}m) \ge N/(5^d m)$.
    Then the contribution from the point $\bp$ satisfies $w_\bp \cdot \sum_{\bq: \bq \ge \bp} g_{\bk, \bq} \ge N/(5^d m) \cdot w_\bp$.
    Let $\text{PGC}(\cG_\bk)$ denote the subset of $P$ such that every point in $\text{PGC}(\cG_\bk)$ is contained in a positively good cell of $\cG_\bk$, we have
    \[
    \bg_\bk^\top B\bw
    \ge \sum_{\bp\in \text{PGC}(\cG_\bk)} w_\bp \cdot \frac{N}{5^d m}.
    \]
    Since each point in $P$ is contained in a good cell for at least $\abs{K}/(2^{d+1}e)$ grids, it follows that
    \begin{align}
    (G\bu)^\top B\bw
    &= \sum_{\bk \in K} \bg_\bk^\top B\bw \notag\\
    &\ge \sum_{\bk \in K} \sum_{\bp \in \text{PGC}(\cG_\bk)} w_\bp \cdot \frac{N}{5^d m} \notag\\
    &= \frac{N}{5^d m} \cdot \sum_{\bp \in P: w_\bp > 0} \sum_{\bk \in \text{PGC}(\bp)} w_\bp \notag\\
    &\ge \frac{N}{5^d m} \cdot \frac{|K|}{2^{d+1}e} \cdot \sum_{\bp \in P: w_\bp > 0} w_\bp \notag\\
    &\ge \frac{N}{5^d m} \cdot \frac{|K|}{2^{d+1}e} \cdot \frac{1}{2} \label{eq:rectangles-contribution-lower-bound}
    \end{align}
    where $\text{PGC}(\bp)$ denotes the subset of $K$ such that the point $\bp$ is contained in a positively good cell of $\cG_\bk$ for every $\bk \in \text{PGC}(\bp)$, and the last inequality holds because of the assumption in \cref{eq:rectangles-assumption-positive-weight}.
    Combining \cref{eq:rectangles-Gu-upper-bound} and \cref{eq:rectangles-contribution-lower-bound}, we have
    \[
    \frac{N}{5^d m} \cdot \frac{|K|}{2^{d+1}e} \cdot \frac{1}{2}
    \le (G\bu)^\top B\bw
    \le \norm{G\bu}_2 \cdot \norm{B\bw}_2
    \le \sqrt{N \cdot |K|} \cdot \sqrt{N} \cdot \norm{B\bw}_\infty,
    \]
    where the second inequality holds by Cauchy-Schwarz inequality, and the last inequality holds because $\norm{B\bw}_2 \le \sqrt{N} \cdot \norm{B\bw}_\infty$.
    Substituting $|K|= \Omega_d(\log^{d-1} m)$ and $n \le m/(2^{d+1}e) < 2n$ into the inequality, we get $\norm{B\bw}_\infty = \Omega_d\left(\log^{(d-1)/2} n/n\right), $ which concludes the proof.

\end{proof}

With this lemma in hand, we can prove the following lower bound for the $\gamma$-VC dimension of $\R_d$.

\begin{theorem}[\cref{thm:intro-rectangles-lower} restated]
    \label{thm:rectangles-lower-bound}
    For every fixed integer $d \ge 2$ and $\gamma \in (0, 1]$, $\VC_\gamma(\R_d) = \Omega_d(\log^{(d-1)/2}(1/\gamma)/\gamma)$.
\end{theorem}

\begin{proof}[Proof of \cref{thm:rectangles-lower-bound}]
    Let $P$ be a point set of size $n$ in $[0, 1]^d$ as guaranteed by \cref{lem:rectangles-good-point-set}.
    According to \cref{lem:rectangles-large-infinity-norm}, $\norm{B\bw}_\infty = \Omega_d(\log^{(d-1)/2} n/n)$, which implies that for any distribution $\cD$ over $P$, and any coloring $c\colon P \to \{\pm 1\}$, there exists a $d$-dimensional rectangle $R$ such that
    \[
    \abs*{\sum_{\bp \in R}w_\bp}
    = \norm{B\bw}_\infty
    = \Omega_d\left(\frac{\log^{(d-1)/2} n}{n}\right),
    \]
    and $\sum_{\bp \notin R}w_\bp$ is 0 or has a different sign with $\sum_{\bp \in R}w_\bp$.
    Otherwise, if $\sum_{\bp \notin R}w_\bp$ has the same sign with $\sum_{\bp \in R}w_\bp$, then the unit box $[0, 1]^d$ satisfies
    \[
    \abs*{\sum_{\bp \in P}w_\bp}
    = \abs*{\sum_{\bp \in R}w_\bp + \sum_{\bp \notin R}w_\bp}
    > \abs*{\sum_{\bp \in R}w_\bp},\] contradicting the definition of $\norm{B\bw}_\infty$.
    Therefore, if $\sum_{\bp \in R}w_\bp$ is positive, then $\sum_{\bp \notin R}w_\bp$ is non-positive, then the corresponding hypothesis $h_R$ satisfies
    \[
    \corr_\cD(h_R)
    = \bbE_{\cD}[h_R(\bp) \cdot c(\bp)]
    = \sum_{\bp \in R}w_\bp - \sum_{\bp \notin R}w_\bp
    \ge \sum_{\bp \in R}w_\bp
    = \Omega_d\left(\frac{\log^{(d-1)/2} n}{n}\right),
    \]
    and $-h_R$ satisfies $\corr_\cD(-h_R) = \Omega_d(\log^{(d-1)/2} n/n) $ for the case $\sum_{\bp \in R}w_\bp$ is negative.
    Now, choose $n=\Theta_d( \log^{(d-1)/2}(1/\gamma)/\gamma))$ with a sufficiently small constant. Then the preceding bound implies that either $\corr_\cD(h_R)$ or $\corr_\cD(-h_R)$ is at least $\gamma$. Thus $P$ is $\gamma$-interpolated by $\R_d$, which concludes the proof.
\end{proof}

\section{Weak-to-Strong Learning}
\label{sec:baggedada}
In the worst case over base classes of VC dimension $d$, weak-to-strong learning has a sample complexity of $\Theta\left(d/(\eps\gamma^2) + \log(1/\delta)/\eps\right)$.
A natural question is: if the $\gamma$-VC dimension of the base hypothesis class is $d_\gamma$, can we establish a similar result where the sample complexity depends on $d_\gamma$ instead of $d$?

We answer this question by proving the following result, which shows that combining AdaBoost with bagging~\citep{baggingLarsen} achieves a sample complexity of $O\left((d_{C\gamma} + \log(1/\delta))/\eps\right)$ when we have a $\gamma$-weak learner and the base hypothesis class has $C\gamma$-VC dimension $d_{C\gamma}$ for a sufficiently small constant $C\in (0, 1)$. Besides, we also establish a lower bound of $\Omega\left((d_\gamma + \log(1/\delta))/\eps\right)$

\subsection{Upper bound}

\begin{theorem}[\cref{thm:algo} restated]\label{thm:weak_to_strong-upper_bound}
    There is a universal constant $C \in (0, 1)$ such that the following holds.
    Suppose $\cW$ is a $\gamma$-weak learner with base class $\cH$, and let $d_{C\gamma} = \VC_{C\gamma}(\cH)$.
    Then the {\upshape\textsc{BaggedAdaBoost}} algorithm, using $\cW$, satisfies the following guarantee. For every distribution $\cD$ over $\cX \times \{\pm 1\}$ that is $\gamma$-realizable by $\cH$, and every $\delta \in (0, 1)$, given $n$ i.i.d. samples from $\cD$, the algorithm returns a classifier with error at most
    \[
        O\left(
            \frac{d_{C\gamma} + \log(1/\delta)}{n}
        \right)
    \]
    with probability at least $1 - \delta$. Equivalently, error at most $\eps$ is achieved with $O\left(\frac{d_{C\gamma} + \log(1/\delta)}{\eps}\right)$ samples.
\end{theorem}

\citet{larsen2022optimal} proposed the first weak-to-strong learning algorithm with optimal sample complexity, combining a variant of AdaBoost with Hanneke's subsampling procedure~\citep{hanneke2016optimal}. \citet{baggingLarsen} later showed that the subsampling procedure can be replaced by bagging, yielding the simpler and more practical \textsc{BaggedAdaBoost} algorithm. Both analyses reduce the strong-learning guarantee to the constant-error statement of \citep[Theorem~4]{larsen2022optimal}: it suffices to show that, given $O\!\left(\gamma^{-2}\VC(\cH)+\log(1/\delta)\right)$ samples, every voting classifier in $\sign(\conv(\cH))$ that achieves margin at least $\gamma$ on every sample has generalization error at most $1/200$ under the underlying distribution. Our main technical contribution is to show that the dependence on $\gamma^{-2}\VC(\cH)$ in this statement can be replaced by $\VC_{C\gamma}(\cH)$ for a universal constant $C\in(0,1)$. 

We therefore focus on establishing this refined constant-error guarantee. \cref{thm:weak_to_strong-upper_bound} follows from the bagging analysis of~\citet{baggingLarsen}, replacing their Corollary~15 with \cref{thm:weak_to_strong-constant_error} below.

\begin{theorem}\label{thm:weak_to_strong-constant_error}
    Let $\gamma \in (0, 1]$.
    Let $\cH \subseteq \{\pm 1\}^\cX$ be a hypothesis class with ${C \gamma}$-VC dimension $d_{C \gamma} \ge 1$ where $C\in (0,1)$ is a constant.
    Let $\cD$ be an arbitrary distribution over $\cX \times \{\pm 1\}$.
    Let $S \sim \cD^n$ be a sample of size $n$.
    If $n \ge C_0 \cdot (d_{C \cdot\gamma} + \log(1/\delta))$ for a universal constant $C_0$, then with probability at least $1 - \delta$ over the choice of $S$, any voting classifier $\hat{f}(x) = \sign(f(x))$ with $f \in \conv(\cH)$ that satisfies $yf(x) \ge \gamma$ for all $(x, y) \in S$ also satisfies
    \[
    \cL_{\cD}(\hat{f})
    = \bbP_{x \sim \cD}[\hat{f}(x) \neq y]
    \le \frac{1}{200}
    .
    \]
\end{theorem}

The proof of \cref{thm:weak_to_strong-constant_error} follows that of~\citet[Theorem 4]{larsen2022optimal}, and the key difference is that we use the $C \gamma$-VC dimension instead of the VC dimension to bound the size of a certain set of classifiers, as shown in \cref{lem:weak_to_strong-size_hypotheses}.
We first introduce some notations.

Since $f \in \conv(\cH)$, we can write $f$ as a convex combination of some hypotheses in $\cH$, i.e., $f(x) = \sum_h w_h h(x)$ where $\sum_{h} w_h = 1$ and $w_h \ge 0$ for all $h \in \cH$.
Let $\cD_f$ be the distribution over $\cH$ where $h$ has probability $w_h$.
Draw $t$ i.i.d. hypotheses $h_1, \dotsc, h_t$ from $\cD_f$ and then throw away each $h_i$ independently with probability $1/2$.
Let $h'_1, \dotsc, h'_{t'}$ be the remaining hypotheses, and let $g = (1/t') \sum_{i=1}^{t'} h'_i$.
Denote the distribution of $g$ by $\cD_{f, t}$, and define the expected error of $g$ over its internal randomness as $\cL_\cD(f, t) = \bbP_{(x, y) \sim \cD, g \sim \cD_{f, t}}[yg(x) \le 0]$, and the empirical error of $g$ over $S$ as $\cL_S(f, t) = \bbP_{(x, y) \in S, g \sim \cD_{f, t}}[yg(x) \le 0]$.
Besides, let $\cL_\cD(f) = \bbP_{(x, y) \sim \cD}[yf(x) \le 0]$.

We recall that if $t$ is sufficiently large, then $\cL_\cD(f)$ can be approximated by $\cL_\cD(f, t)$, and if $f$ has margin at least $\gamma$ on $S$, then $\cL_S(f, t)$ is small.
The following lemmas establish these properties.

\begin{lemma}[{\citealp[Lemma 3]{larsen2022optimal}}]\label{lem:weak_to_strong-approximation}
    Let $\cH \subseteq \{\pm 1\}^\cX$ be a hypothesis class, and $\cD$ be an arbitrary distribution over $\cX \times \{\pm 1\}$.
    For any $f \in \conv(\cH)$ and any $t \ge 36$, we have $\cL_\cD(f) \le 3\cL_\cD(f, t)$.
\end{lemma}

\begin{lemma}[{\citealp[Lemma 4]{larsen2022optimal}}]\label{lem:weak_to_strong-small_empirical_error}
    Let $\gamma \in (0, 1]$.
    Let $S$ be a sample of size $n$ in $\cX \times \{\pm 1\}$, and $f \in \conv(\cH)$ satisfies $yf(x) \ge \gamma$ for all $(x, y) \in S$.
    If $t \ge 1024 \gamma^{-2}$, we have $\cL_S(f, t) \le 1/1200$.
\end{lemma}

Thus, the key step is to show that $\cL_\cD(f, t)$ can be approximated by $\cL_S(f, t)$ for all $f \in \conv(\cH)$.
\cref{lem:weak_to_strong-uniform_convergence} establishes this uniform convergence result.

\begin{lemma}\label{lem:weak_to_strong-uniform_convergence}
    Let $\gamma \in (0, 1]$ and $t = 10^{34}\gamma^{-2}$.
    Let $\cH \subseteq \{\pm 1\}^\cX$ be a hypothesis class with $C\gamma$-VC dimension $d_{C \gamma} \ge 1$ where $C\in (0,1)$ is a constant.
    Let $\cD$ be an arbitrary distribution over $\cX \times \{\pm 1\}$.
    Let $S \sim \cD^n$ be a sample of size $n$.
    If $n \ge C_0 \cdot (d_{C\gamma} + \log(1/\delta))$ for a universal constant $C_0$, then we have
    \[
    \bbP_{S \sim \cD^n}\left[\sup_{f \in \conv(\cH)}|\cL_\cD(f, t) - \cL_S(f, t)| > \frac{1}{1200}\right]
    < \delta
    .
    \]
\end{lemma}

Before proving \cref{lem:weak_to_strong-uniform_convergence}, we first show how to use it together with the above lemmas to prove \cref{thm:weak_to_strong-constant_error}.

\begin{proof}[Proof of \cref{thm:weak_to_strong-constant_error}]
    Setting $t = 10^{34}\gamma^{-2}$, by the above lemmas, we have with probability at least $1 - \delta$ over the choice of $S$,
    \begin{align*}
    \cL_\cD(\hat{f})
    \le \cL_\cD(f)
    \le 3\cL_\cD(f, t)
    \le 3 \cdot (\cL_S(f, t) + \frac{1}{1200})
    \le 3 \cdot (\frac{1}{1200} + \frac{1}{1200})
    = \frac{1}{200}
    ,
    \end{align*}
    where the four inequalities follows respectively from the definition of $\cL_\cD(f)$, \cref{lem:weak_to_strong-approximation}, \cref{lem:weak_to_strong-uniform_convergence}, and \cref{lem:weak_to_strong-small_empirical_error}.
    This completes the proof.
\end{proof}

\paragraph{Relating $\cL_S(f, t)$ to $\cL_\cD(f, t)$.}
What remains is to prove \cref{lem:weak_to_strong-uniform_convergence}, namely relating $\cL_S(f, t)$ to $\cL_\cD(f, t)$.
We first introduce a ghost sample $S'$ that is independent of $S$ and drawn i.i.d. from $\cD^n$.
Then by a standard symmetrization argument, we can relate the uniform convergence of $\cL_S(f, t)$ to $\cL_\cD(f, t)$ to the uniform convergence of $\cL_S(f, t)$ to $\cL_{S'}(f, t)$.

\begin{lemma}[{\citealp[Lemma 6]{larsen2022optimal}}]\label{lem:weak_to_strong-ghost_sample}
    Let $\cD$ be a distribution over $\cX \times \{\pm 1\}$.
    Let $S, S'$ be two independent samples of size $n$ drawn i.i.d. from $\cD$.
    For any integer $t \ge 1$ and any $f \in \conv(\cH)$, if $n \ge 2400^2$, we have
    \[
    \bbP_{S \sim \cD^n}\left[\sup_{f \in \conv(\cH)}|\cL_S(f, t) - \cL_\cD(f, t)| > \frac{1}{1200}\right]
    \le 2\cdot \bbP_{S, S' \sim \cD^n}\left[\sup_{f \in \conv(\cH)}|\cL_S(f, t) - \cL_{S'}(f, t)| > \frac{1}{2400}\right]
    .
    \]
\end{lemma}

To bound the right-hand side, we let $P$ be a sample sequence of size $2n$ drawn i.i.d. from $\cD^{2n}$, and then draw a subsequence $S$ of size $n$ uniformly at random from $P$ without replacement, and let $S'$ be the remaining subsequence.
Then $S$ and $S'$ have the same distribution as two independent samples drawn i.i.d. from $\cD^n$.
From now on, we view $S$ and $S'$ as drawn from $P$ in this way.

For any sample sequence $P$ in the support of $\cD^{2n}$, define $\cF_{\mu, \rho}(P)$ as the set of all $f \in \conv(\cH)$ for which $\bbP_{(x, y) \sim P}[\abs{f(x)}<\mu]\leq \rho$.
And define $\hat{\cF}_{\mu, \rho}(P) = \{\sign(f)_{|P}: f \in \cF_{\mu, \rho}(P)\}$ to be the set of restrictions to $P$ of the classifiers induced by $\cF_{\mu, \rho}(P)$. Thus, $\hat{\cF}_{\mu, \rho}(P)$ has a finite size.
The following lemma shows that the uniform convergence of $\cL_S(f, t)$ to $\cL_{S'}(f, t)$ can be related to the size of $\hat{\cF}_{\mu, \rho}(P)$.

\begin{lemma}[{\citealp[Lemma 8]{larsen2022optimal}}]\label{lem:weak_to_strong-ghost-size_hypotheses}
    Let $\rho \in (0, 1)$ and $t \ge 9600^{2}\rho^{-2}$.
    For any data sequence $P$ in the support of $\cD^{2n}$, if $1/t \le \mu \le \rho/(9600\sqrt{t})$, we have
    \[
    \bbP_{S, S' \sim \cD^n}\left[\sup_{f \in \conv(\cH)}|\cL_S(f, t) - \cL_{S'}(f, t)| > \frac{1}{2400}\right]
    \le \sup_{P \in \supp(\cD^{2n})} 2 \cdot |\hat{\cF}_{\mu, \rho}(P)| \cdot \exp(-2n/9600^2)
    .
    \]
\end{lemma}

The last step is to bound the size of $\hat{\cF}_{\mu, \rho}(P)$ by the $C\gamma$-VC dimension of $\cH$, instead of the VC dimension of $\cH$ as in~\citet{larsen2022optimal}.

\begin{lemma}\label{lem:weak_to_strong-size_hypotheses}
    Let $\mu \in (0, 1]$ and $\rho \in (0, 1)$.
    Let $\cH$ be a hypothesis class with $\mu$-VC dimension $d_\mu$.
    For any data sequence $P = ((x_1, y_1), \dotsc, (x_{2n}, y_{2n}))$ in the support of $\cD^{2n}$, we have $|\hat{\cF}_{\mu, \rho}(P)| \le (e/\rho)^{d_\mu+4\rho n}$.
\end{lemma}
\begin{proof}[Proof of \cref{lem:weak_to_strong-size_hypotheses}]
    By definition of $\cF_{\mu, \rho}(P)$, for any $f \in \cF_{\mu, \rho}(P)$, we have at most $\rho$ fraction of the points in $P$ that satisfy $|f(x)| < \mu$.
    Thus, we can split $\cF_{\mu, \rho}(P)$ into at most $\binom{2n}{\le 2\rho n}$ subsets, where each subset corresponds to a specific choice of at most $2\rho n$ data points in $P$.
    In particular, for any $P' \subseteq P: |P'| \le 2\rho n$, we define
    \[
    \cF_{\mu, P'}(P)
    = \{f \in \conv(\cH): |f(x)| < \mu \text{ for all } (x, y) \in P', \text{ and } |f(x)| \ge \mu \text{ for all } (x, y) \in P\setminus P'\}
    ,
    \]
    and $\hat{\cF}_{\mu, P'}(P) = \{\sign(f)_{|P}: f \in \cF_{\mu, P'}(P)\}$ accordingly.
    Then we have $\hat{\cF}_{\mu, \rho}(P) = \cup_{P' \subseteq P: |P'| \le 2\rho n} \hat{\cF}_{\mu, P'}(P)$.

    Fix any $P' \subseteq P$ with $|P'| \le 2\rho n$.
    We analyse the VC dimension of $\hat{\cF}_{\mu, P'}(P)$, and use Sauer-Shelah-Perles lemma to bound its size.
    % We first study the maximum size of a subsequence of $P\setminus P'$ that can be shattered by $\hat{\cF}_{\mu, P'}(P)$.
    Let $S = (x_1, \dotsc, x_{|S|})$ be a subsequence of input points in $P\setminus P'$ that is shattered by $\hat{\cF}_{\mu, P'}(P)$.
    Then fix any binary labeling $c$, there exists a convex combination $f \in \cF_{\mu, P'}(P)$ whose corresponding voting classifier $\hat{f} \in \hat{\cF}_{\mu, P'}(P)$ realizes this labeling, i.e., $\hat{f}(x) = c(x)$ for all $(x, c(x)) \in S$.
    The definition of $\cF_{\mu, P'}(P)$ implies that $|f(x)| \ge \mu$ for all $(x, y) \in P\setminus P'$, and thus for any distribution $\cQ$ over $S$, we have
    \begin{align*}
    \corr_\cQ(f)
    &= \bbE_{x \sim \cQ}[f(x)c(x)]\\
    &= \bbE_{x \sim \cQ}[\hat{f}(x)c(x) \cdot |f(x)|]\\
    &\ge \mu \cdot \bbE_{x \sim \cQ}[\hat{f}(x)c(x)]\\
    &= \mu
    .
    \end{align*}
    Since $f$ is a convex combination of some hypotheses in $\cH$, we have $\mu \le \corr_\cQ(f) \le \max_{h \in \cH} \corr_\cQ(h)$.
    Thus, for any distribution and any labeling of $S$, there exists a hypothesis in $\cH$ that has correlation at least $\mu$ with the labeling.
    This implies that $S$ is $\mu$-interpolated by $\cH$.
    By definition of $\mu$-VC dimension, we have $|S| \le \VC_\mu(\cH) = d_\mu$.

    Therefore, the VC dimension of $\hat{\cF}_{\mu, P'}(P)$ is at most $\VC(\hat{\cF}_{\mu, P'}(P)) \le d_\mu + 2\rho n$, where the additional $2\rho n$ term is due to the fact that there are at most $2\rho n$ data points in $P'$ that can be shattered additionally. When $d_\mu + 2\rho n\geq 2n$, we trivially have $\abs*{\hat{\cF}_{\mu, \rho}(P)} \le 2^{2n}\leq (e/\rho)^{d_\mu+4\rho n}$.
    Otherwise, by Sauer-Shelah-Perles lemma, we have
    \begin{align*}
        \abs*{\hat{\cF}_{\mu, \rho}(P)}
        &=\abs*{ \cup_{P' \subseteq P: |P'| \le 2\rho n} \hat{\cF}_{\mu, P'}(P)}\\
        &\le \binom{2n}{\le 2\rho n} \cdot \sup_{P' \subseteq P: |P'| \le 2\rho n} |\hat{\cF}_{\mu, P'}(P)|\\
        &\le \left(\frac{2en}{2\rho n}\right)^{2\rho n} \cdot \left(\frac{2en}{d_\mu + 2\rho n}\right)^{d_\mu+2\rho n}\\
        &\le (e/\rho)^{d_\mu+4\rho n}
        .
    \end{align*}
\end{proof}

With these lemmas in hand, we are ready to prove \cref{lem:weak_to_strong-uniform_convergence}.
\begin{proof}[Proof of \cref{lem:weak_to_strong-uniform_convergence}]
    When $n \ge 2400^2$, $t \ge 9600^2\rho^{-2}$, and $1/t \le \mu \le \rho/(9600\sqrt{t})$, these lemmas together imply that
    \begin{align*}
    &\bbP_{S \sim \cD^{n}}\left[\sup_{f \in \conv(\cH)}|\cL_S(f, t) - \cL_\cD(f, t)| > \frac{1}{1200}\right]\\
    \text{(\cref{lem:weak_to_strong-ghost_sample})}
    &\le 2\cdot \bbP_{S, S' \sim \cD^{n}}\left[\sup_{f \in \conv(\cH)}|\cL_S(f, t) - \cL_{S'}(f, t)| > \frac{1}{2400}\right]\\
    \text{(\cref{lem:weak_to_strong-ghost-size_hypotheses})}
    &\le \sup_{P \in \supp(\cD^{2n})} 4 \cdot |\hat{\cF}_{\mu, \rho}(P)| \cdot \exp(-2n/9600^2)\\
    \text{(\cref{lem:weak_to_strong-size_hypotheses})}
    &\le 4 \cdot (e/\rho)^{d_\mu+4\rho n} \cdot \exp(-2n/9600^2)
    .
    \end{align*}

    Setting $\rho = 10^{-12}, t = 10^{10}\rho^{-2}\gamma^{-2} = 10^{34}\gamma^{-2}, \mu = \rho/(10000\sqrt{t}) = 10^{-33}\gamma$, we have
    \begin{align*}
    4 \cdot (e/\rho)^{d_\mu+4\rho n} \cdot \exp(-2n/9600^2)
    &\le \exp(2 + \log(e/\rho)(d_\mu + 4\rho n) - 2n/9600^2)\\
    &\le \exp(2 + 42d_\mu - 10^{-8}n    )
    .
    \end{align*}
    Setting $n \ge C_0 \cdot (d_\mu + \log(1/\delta))$ for a sufficiently large universal constant $C_0 \ge 10^{10}$, we have $\exp(2 + 42d_\mu - 10^{-8}n) \le \delta$.
    This completes the proof with $C = 10^{-33}$.
\end{proof}

\subsection{Lower bound} \label{sec:weak_to_strong-lower_bound}

In this section, we show that the sample complexity of weak-to-strong learning with a $\gamma$-weak learner is at least $\Omega\left((d_{\gamma} + \log(1/\delta))/\eps\right)$.
The proof follows the standard lower bound argument for realizable PAC learning due to~\citet{blumer1989learnability}.

\begin{theorem}\label{thm:weak_to_strong-lower_bound}
    Let $\delta \in (0, 1/14)$, $\eps \in (0, 1/16)$ and $\gamma \in (0, 1]$.
    Let $\cH \subseteq \{\pm 1\}^\cX$ have $\gammaVC(\cH) = d_\gamma \ge 2$, and let $\cA$ be a weak-to-strong learning algorithm that uses a $\gamma$-weak learner with base hypothesis class $\cH$.
    Let $n$ be an integer with $n\leq \max\{\frac{d_\gamma}{128\eps},\frac{\log(1/2\delta)}{4\eps}\}$.

    Then there exists a distribution $\cD$ over $\cX$ and a target $c:\cX\to \{\pm 1\}$ such that the distribution $(x,c(x))$, $x\sim \cD$, is $\gamma$-realizable by $\cH$, and let $S \sim \cD^n$ be a sample of size $n$, then
    \[
    \bbP[\cL_\cD(\cA(S))>\eps]\geq \delta.
    \]
\end{theorem}
\begin{proof}[Proof of \cref{thm:weak_to_strong-lower_bound}]
    We handle the two terms in the maximum separately. In both cases, the distribution $\cD$ is supported on a set $\cX'$ with $|\cX'| \leq d_\gamma$ that is $\gamma$-interpolated by $\cH$. Thus, any labeled distribution over $\cX'$ is $\gamma$-realizable by $\cH$, and every call to the weak learner is valid. By the learning model, $\cA(S)$ depends only on the labeled sample $S$ and its internal randomness, and thus is independent of the labels of unseen points. 

    \paragraph{Case 1: $n \le d_\gamma/(128\eps)$.} Let $\cX' = \{x_1, \dotsc, x_{d_\gamma}\}\subseteq \cX$ $\gamma$-interpolated by $\cH$.
    Let $\cD$ give probability $1 - 8\eps$ to $x_1$ and probability $8\eps/(d_\gamma - 1)$ to each of the remaining $x_2, \dotsc, x_{d_\gamma}$. 
    Draw $c$ uniformly from $\cC$, where $\cC$ is the set of all labelings of $\cX'$ that label $x_1$ as $1$.
    Let $S = ((s_1, c(s_1)), \dotsc, (s_{n}, c(s_{n})))$ be a training sequence of size $n$, where $s_1, \dotsc, s_n$ are drawn i.i.d. from $\cD$.

    Let $X$ be a random variable that counts the number of times $\{x_2, \dotsc, x_{d_\gamma}\}$ appears in $S$.
    Then, by linearity of expectation, we have
    \[
    \bbE[X]
    = \sum_{i=1}^n \bbP_{s_i \sim \cD}[s_i \in \{x_2, \dotsc, x_{d_\gamma}\}]
    = n \cdot (d_\gamma - 1) \cdot \frac{8\eps}{d_\gamma - 1}
    \le d_\gamma/16
    ,
    \]
    where the last inequality is by the assumption that $n \le d_\gamma/(128\eps)$.
    By Markov's inequality, we have $\bbP[X \ge d_\gamma/4] \le 1/4$.
    Thus, with probability at least $3/4$ over the choice of $S$, we have $X < d_\gamma/4$.

    Conditioned on $X < d_\gamma/4$, there are at least $3d_\gamma/4 - 1$ points in $\{x_2, \dotsc, x_{d_\gamma}\}$ that do not appear in $S$, and let $I$ be the set of such points.
    Since $c$ is uniformly at random chosen from $\cC$, the labels of these points are independent and uniformly at random chosen from $\{\pm 1\}$.
    Thus, over the randomness of $c$, any classifier $\cA(S)$ has error at least $1/2$ on each of these points.

    Let $\cL'_\cD(\cA(S), c)$ be the expected error of $\cA(S)$ with respect to $\cD$ and $c$ that only considers the points in $\{x_2, \dotsc, x_{d_\gamma}\}$, i.e.,
    \[
    \cL'_\cD(\cA(S), c)
    = \sum_{x_i \in \{x_2, \dotsc, x_{d_\gamma}\}} \cD(x_i) \cdot \ind\{c(x_i) \neq \cA(S)(x_i)\}
    .
    \]
    Over the randomness of $c$ and $S$, the error of $\cA(S)$ satisfies
    \begin{align*}
    \bbE_{c \sim \cC,S \sim \cD^n}[\cL'_\cD(\cA(S), c)]
    &\ge \bbE_{S}\left[\ind\{X < d_\gamma/4\} \cdot \sum_{x_i \in I}\cD(x_i) \cdot \bbP_{c}[c(x_i) \neq \cA(S)(x_i)]\right]\\
    &\ge \frac{3}{4} \cdot \frac{8\eps}{d_\gamma - 1} \cdot (\frac{3d_\gamma}{4} - 1) \cdot \frac{1}{2} \\
    &\ge 3\eps/2
    ,
    \end{align*}
    where the last inequality is by the assumption that $d_\gamma \ge 2$.
    Thus, there exists a choice of $c$ such that $\bbE_{S \sim \cD^n}[\cL'_\cD(\cA(S), c)] \ge 3\eps/2$.
    Since $\cL'_\cD(\cA(S), c)$ is at most $8\eps$, reverse Markov's inequality implies that for such a choice of $c$, we have
    \[
    \bbP_{S \sim \cD^n}[\cL_\cD(\cA(S)) > \eps]
    \ge \bbP_{S \sim \cD^n}[\cL'_\cD(\cA(S), c) > \eps]
    \ge \frac{3\eps/2 - \eps}{8\eps - \eps}
    = \frac{1}{14}
    > \delta
    ,
    \]
    which completes the proof of this part.

    \paragraph{Case 2: $n \le \log(1/(2\delta))/(4\eps)$.}
    Let $\cX' = \{x_1, x_2\}$ be $\gamma$-interpolated by $\cH$.
    Let $\cD$ give probability $1 - 2\eps$ to $x_1$ and probability $2\eps$ to $x_2$.
    Draw $c$ uniformly from $\cC = \{c_1, c_2\}$, where $c_1(x_1) = c_1(x_2) = c_2(x_1) = 1$ and $c_2(x_2) = -1$.
    
    Probability that $x_2$ does not appear in $S$ is $(1 - 2\eps)^n$.
    Since $c$ is uniformly at random chosen from $\cC$, the label of $x_2$ is independent and uniformly at random chosen from $\{\pm 1\}$.
    Thus, over the randomness of $c$ and $S$, we have
    \begin{align*}
    \bbP_{c \sim \cC,S \sim \cD^n}[\cL_\cD(\cA(S)) > \eps]
    &\ge \bbE_{S \sim \cD^n}[\ind\{x_2 \text{ does not appear in } S\} \cdot \bbP_{c \sim \cC}[c(x_2) \neq \cA(S)(x_2)]] \\
    &\ge (1 - 2\eps)^n \cdot \frac{1}{2}
    ,
    \end{align*}
    which implies that there exists a choice of $c$ such that
    \[
    \bbP_{S \sim \cD^n}[\cL_\cD(\cA(S)) > \eps]
    \ge (1 - 2\eps)^n \cdot 1/2
    \ge e^{-4\eps n}/2
    \ge \delta,
    \]
    where the last two inequalities come from $\eps\in(0, 1/16)$ and $n \le \log(1/(2\delta))/(4\eps)$.
    This completes the proof of this part.

\end{proof}

\section*{Acknowledgements}
\paragraph{AI Disclosure.} 
ChatGPT suggested using the discrepancy bound for intervals in permutations due to~\citet{spencer2001discrepancy}, which led to an improvement of the upper bound in \cref{thm:stumps-upper} from $O\!\left(\frac{\sqrt d}{\gamma}\log^{3/2}\!\left(\frac d\gamma\right)\right) $ to $O\!\left(\frac{\sqrt d}{\gamma}\log\!\left(\frac d\gamma\right)\right)$.
The original bound was derived by the authors using $\gamma_2$-norm techniques from~\citet{matouvsek2020factorization}. 
ChatGPT also suggested the argument used in \cref{lem:signed-discrepancy-bound}. 
ChatGPT and Claude were additionally used for calculations, verification, and language editing. 
The authors take full responsibility for all results and the content of this paper.

% \newpage
\addcontentsline{toc}{section}{References}
\bibliographystyle{plainnat}
\bibliography{references}

% Keep these appendix drafts disabled until the co-authors approve their proofs.
% The uniform scale comparison has an unresolved correctness objection.
% \newpage
% \appendix
% \phantomsection
% \addcontentsline{toc}{section}{Appendices}
% \clearpage
% \input{content/scale_stability}
% \clearpage
% \input{content/scale_comparison_bounds}

\end{document}